\documentclass[twoside,11pt]{article}

\usepackage[preprint]{jmlr2e}
\usepackage{hanch_sty}
\usepackage{bbold,amsmath,amsfonts,amssymb}
\usepackage{enumitem}
\usepackage{balance}
\newtheorem{innercustomthm}{Theorem}
\newenvironment{customthm}[1]
  {\renewcommand\theinnercustomthm{#1}\innercustomthm}
  {\endinnercustomthm}
 \newtheorem{innercustomlem}{Lemma}
\newenvironment{customlem}[1]
  {\renewcommand\theinnercustomlem{#1}\innercustomlem}
  {\endinnercustomlem}
\newtheorem{lemma_sec}{Lemma}[section]

\newtheorem{problem}{Problem}
\usepackage{multirow}
\usepackage{subfig}
\usepackage{caption}

\usepackage{ifthen}
\newboolean{showcomments}
\setboolean{showcomments}{false}
\usepackage{color}
\usepackage{todonotes}

\definecolor{bleudefrance}{rgb}{0.19, 0.55, 0.91}
\definecolor{ao(english)}{rgb}{0.0, 0.5, 0.0}

\newcommand{\addcite}[0]{\ifthenelse{\boolean{showcomments}}
{\textcolor{purple}{(add cite(s)) }}{}}%

\newcommand{\enrique}[1]{  \ifthenelse{\boolean{showcomments}}
{\todo[inline,color=bleudefrance]{Enrique: #1}}{}}
\newcommand{\rene}[1]{  \ifthenelse{\boolean{showcomments}}
{\todo[inline,color=cyan]{Ren\'e: #1}}{}}
\newcommand{\emmargin}[1]{\ifthenelse{\boolean{showcomments}}{\marginpar{\color{bleudefrance}\tiny EM: #1}}{}}
\newcommand{\hancheng}[1]{  \ifthenelse{\boolean{showcomments}}
{\todo[inline,color=orange]{Hancheng: #1}}{}}

\newcommand{\hl}[1]{\ifthenelse{\boolean{showcomments}}
{\textcolor{red}{#1}}{#1}}

\newboolean{showedits}
\setboolean{showedits}{false}
\usepackage[markup=underlined]{changes}
\definechangesauthor[color=bleudefrance]{EM}
\newcommand{\aem}[1]{
\ifthenelse{\boolean{showedits}}
{\added[id=EM]{#1}}
{\!#1\hspace{-4.75pt}}
}
\newcommand{\repem}[2]{
\ifthenelse{\boolean{showedits}}
{\replaced[id=EM]{#1}{#2}}
{\!#1\hspace{-4.75pt}}
}
\newcommand{\dem}[1]{
\ifthenelse{\boolean{showedits}}
{\deleted[id=EM]{#1}}
{}
}

\DeclareMathOperator{\tr}{tr}
\DeclareMathOperator{\rank}{rank}

\ShortHeadings{Convergence and Implicit Bias of Overparametrized Linear Networks}{Min, Tarmoun, Vidal and Mallada}
\firstpageno{1}

\begin{document}

\title{Convergence and Implicit Bias of Gradient Flow on Overparametrized Linear Networks\thanks{Preprint}}

\author{\name Hancheng Min\footnotemark[2] \footnotemark[3] \email hanchmin@jhu.edu \\
\name Salma Tarmoun\footnotemark[2] \footnotemark[4] \email starmou1@jhu.edu \\
\name Ren\'e Vidal\footnotemark[2] \footnotemark[5] \email rvidal@jhu.edu\\
\name Enrique Mallada\footnotemark[2] \footnotemark[3]
\email mallada@jhu.edu \\
       \footnotemark[2] \addr Mathematical Institute for Data Science, Johns Hopkins University\\
       \footnotemark[3] \addr Department of Electrical and Computer Engineering, Johns Hopkins University\\
       \footnotemark[4] \addr Department of Applied Mathematics and Statistics, Johns Hopkins University\\
       \footnotemark[5] \addr Department of Biomedical Engineering, Johns Hopkins University
       }

\editor{}

\maketitle

\begin{abstract}
    Neural networks trained via gradient descent with random initialization and without any regularization enjoy good generalization performance in practice despite being highly overparametrized. 
    A promising direction to explain this phenomenon is to study how initialization and overparametrization affect convergence and implicit bias of training algorithms. In this paper, we present a novel analysis of single-hidden-layer linear networks trained under gradient flow, which connects initialization, optimization, and overparametrization. 
    Firstly, we show that the squared loss converges exponentially to its optimum at a rate that depends on the level of imbalance and the margin of the initialization.
    Secondly, we show that proper initialization constrains the dynamics of the network parameters to lie within an invariant set. In turn,  minimizing the loss over this set leads to the min-norm solution. Finally, we show that large hidden layer width, together with (properly scaled) random initialization, ensures proximity to such an invariant set during training, allowing us to derive a novel non-asymptotic upper-bound on the distance between the trained network and the min-norm solution.
\end{abstract}


\enrique{The following comment does not need to be address before uploading to arxiv.
In general, when a journal paper is based on a conference one, it is important to highlight the differences/extensions. We shoudl check that this is indeed needed for JMLR (I'm assuming that you want to submit it there?)
}

\section{Introduction}
    \hancheng{In the introduction, changes are gray-colored}
    Neural networks have shown excellent empirical performance in many application domains such as vision~\citep{krizhevsky2012imagenet,rawat2017deep}, speech~\citep{hinton2012deep,graves2013speech} and video games~\citep{silver2016mastering,vinyals2017starcraft}. Among the many unexplained puzzles behind this success is the fact that gradient descent with random initialization, and without explicit regularization,  enjoys good generalization performance despite being highly overparametrized.
    
    One possible explanation of such phenomenon is the implicit bias or regularization that first order gradient algorithms induce under proper initialization assumptions. For example, in classification tasks, gradient descent on separable data can induce a bias towards the max-margin solution~\citep{soudry2018implicit,ji2019gradient,lyu2019gradient}. Similarly, in regression tasks, it has been shown that (deep) matrix factorization models trained by first order methods yield solutions with low nuclear norm~\citep{gunasekar2017implicit} and low rank~\citep{arora2019implicit}. Along the same vein, \citet{saxe2014exact,Gidel2019} have shown that 
    deep linear networks sequentially learn dominant singular values of the input-output correlation matrix.
    
    Another possible explanation is that, in the Neural Tangent Kernel (NTK) regime, the gradient flow of a randomly initialized infinitely wide neural network can be well approximated by the flow of its linearization at initialization~\cite{jacot2018neural,chizat2019lazy,arora2019exact,arora2019fine}.  In this regime, training infinitely wide neural networks mimics kernel methods. In particular, the NTK flow is constrained to lie on a manifold, which improves generalization performance as discussed in ~\citep{arora2019fine}.
    
    While the aforementioned analysis is quite insightful, it requires assumptions on the model and the initialization that are often disconnected. For example, the implicit bias characterized in~\citep{gunasekar2017implicit,arora2019implicit} requires vanishing initialization,
    while the analysis of convergence of gradient algorithms for linear networks requires balanced~\citep{arora2018optimization,arora2018convergence} or spectral~\citep{saxe2014exact,Gidel2019} initialization. Similarly, the NTK regime~\citep{jacot2018neural,arora2019exact}, requires random initialization and infinitely wide networks, making the non-asymptotic analysis challenging~\citep{arora2019exact}.

    This paper aims to bridge some of these gaps. We present a novel analysis of the gradient flow dynamics of overparametrized single-hidden-layer linear networks, which provides a common set of conditions on initialization that lead to convergence and implicit bias. {\color{black}Specifically, we reveal the explicit role of weights imbalance and weights product on the convergence of linear networks, suggesting a broad set of initial parameter values that lead to exponential convergence. }We further characterize a complementary  condition, based on orthogonality, that enforces the learning trajectory to be constrained within an invariant set whose unique global optimum is the min-norm solution. While our analysis does not require infinite width, vanishing, spectral, or random initialization, we show that our exponential convergence and orthogonality conditions are probably approximately satisfied for wide networks with properly scaled random initialization, jointly leading to a bound on the distance to the min-norm solution. Hence, this paper formally connects initialization, exponential convergence of the optimization task, overparametrization and implicit~bias.
    
    This paper makes the following contributions:
    \begin{enumerate}[leftmargin=4mm]
        \item  {\color{black}In Section \ref{sec:conv_grad_flow}, we show that the convergence of linear networks explicitly depends on: 1) a weight imbalance matrix; and 2) the weights product (end-to-end function). With such observation, we provide two conditions, \emph{sufficient imbalance} and \emph{sufficient margin}, on the intialization, with either of them being sufficient for guaranteeing exponential convergence. Our convergence analysis unifies prior work's assumptions and expands them to a broader set of initial conditions, as discussed in Section~\ref{sec:rel_work}.}
        \item In Section \ref{secc:decomp_net_gen}, we show the existence of a subset of the parameter space defined by an orthogonality condition, which is invariant under gradient flow. All trajectories within this invariant set lead to a unique minimizer (w.r.t. the end-to-end function), which corresponds to the min-norm solution. As a result, initializing the network within this invariant set always yields the min-norm solution upon convergence. 
     
        \item In Section \ref{secc:wide_linear_net}, we further show that by randomly initializing the network weights using $\mathcal N(0,1/h^{2\alpha})$ (where $h$ is the hidden layer width and $1/4<\alpha\leq 1/2$),
        one can approximately satisfy both our sufficient imbalance and orthogonality conditions with high probability. Notably, 
        initializations outside the invariant set require exponential convergence to control their deviation from the set. For linear networks our results also provide a novel non-asymptotic upper-bound on the operator norm distance between the trained network and the min-norm solution.
       \end{enumerate}
       
\subsection{Other Related Work}\label{sec:rel_work}

    \textbf{Convergence of Linear Networks}. Convergence in overparametrized linear networks has been studied for both gradient flow~\citep{saxe2014exact,pmlr-v139-tarmoun21a} and gradient descent~\citep{Gidel2019,arora2018convergence,arora2018optimization}. 
    ~\citet{saxe2014exact,Gidel2019,pmlr-v139-tarmoun21a} analyze the trajectory of network parameters under spectral initialization. For non-spectral initialization, although the fact that the imbalance is conserved under gradient flow has been exploited in~\citet{arora2018convergence,arora2018optimization}, the work studies balanced initialization and exploits the structure conveyed by it to study convergence of the learning dynamics. The analysis of convergence in the imbalanced case was recently studied in~\citet{pmlr-v139-tarmoun21a} for both spectral and non-spectral initializations. {\color{black}For non-spectral initialization, specifically, previous analyses largely rely on specific imbalance structure (For example, small imbalance~\citep{arora2018convergence}, and homogeneous imbalance~\citep{pmlr-v139-tarmoun21a}). Our analysis improves upon prior works by studying general imbalance structures. Particularly, our analysis identifies three key parameters, that quantify gaps and the spread of the spectrum of an imbalance matrix, that affect the rate of convergence of gradient flow.
    \enrique{I just added this sentence but wonder if we should put it in "contributions".}
    
    The summary of the convergence results for linear networks is shown in Table \ref{tb_lin_conv_init}, and we also illustrate all aforementioned non-spectral initialization in Figure \ref{fig_lin_conv_init}.}
    \begin{table}[!h]
    \centering
    \begin{tabular}{c|c|c}
    \hline
     & \emph{Spectral} & \emph{Non-spectral} \\ [0.2cm]  \hline
    \multirow{2}{*}{\emph{Balanced}} &\multirow{2}{*}{\begin{tabular}{c}
          \citep{saxe2014exact} \\
          \citep{Gidel2019}
     \end{tabular}} & \begin{tabular}{c}
          Exactly balanced\\
          \citep{arora2018optimization}
     \end{tabular} \\  [0.4cm]
    & & \begin{tabular}{c}
        Sufficient margin \\
        + Approximately balanced\\
     \citep{arora2018convergence} 
     \end{tabular} \\ [0.4cm] \hline
    \multirow{3}{*}{\emph{Imbalanced}} &\multirow{3}{*}{\begin{tabular}{c}
          \citep{pmlr-v139-tarmoun21a}
     \end{tabular}} &\begin{tabular}{c}
        Homogeneous imbalance \\
        \citep{pmlr-v139-tarmoun21a}
     \end{tabular} \\ [0.4cm]
    & &\begin{tabular}{c}
        \textbf{Sufficient level of imbalance} \\
        (Our work)
     \end{tabular} \\  [0.4cm]
    & &\begin{tabular}{c}
        \textbf{Sufficient margin} \\
        (Our work)
     \end{tabular}\\ \hline
    \end{tabular}
    \caption{List of initialization types that have been studied for the convergence of gradient flow on the single-hidden-layer linear networks. All non-spectral initialization types listed here are illustrated in Figure \ref{fig_lin_conv_init}}
    \end{table}\label{tb_lin_conv_init}
    
    \noindent
    \textbf{Wide Neural Networks}. 
    There has been a rich line of research that studies the 
    convergence~\citep{du2019a,du2019b,du2019width,allen2019convergence} and generalization~\citep{allen2019learning,arora2019fine,arora2019exact,li2018learning,cao2019generalization,buchanan2020deep} of wide neural networks with random initialization. The behavior of such networks in their infinite width limit can be characterized by the \emph{Neural Tangent Kernel} (NTK)~\citep{jacot2018neural}. Heuristically, training wide neural networks can be approximately viewed as kernel regression under gradient flow/descent~\citep{arora2019exact}. Hence, \hl{convergence and generalization can be understood by studying the non-asymptotic results regarding the equivalence of finite width networks to their infinite limit~\citep{du2019a,du2019b,allen2019convergence,arora2019fine,arora2019exact,buchanan2020deep}.}\emmargin{This comment is vague} More generally, such non-asymptotic results are related to the ``lazy training"~\citep{chizat2019lazy,du2019b,allen2019convergence}, where the network weights do not deviate too much from its initialization during training. Our results for wide linear networks presented in Section \ref{secc:wide_linear_net} do not follow the NTK analysis, but provide an alternative view on the effect of random initialization for linear networks when the hidden layer is sufficiently wide. 
    
    \subsection{Notation}
    For a matrix $A$, we let $A^T$ denote its transpose, $\mathrm{tr}(A)$ denote its trace, $\lambda_i(A)$ and $\sigma_i(A)$ denote its $i$-th eigenvalue and $i$-th singular value, respectively, in decreasing order (when adequate). {\color{black}For an $n\by m$ matrix $A$, we let $\sigma_{\min}(A)=\sigma_{\min\{n,m\}}(A)$, and we conventionally let $\lambda_i(A)=\sigma_i(A)=0, \forall i>\min\{m,n\}$.} We let $[A]_{ij}$, $[A]_{i,:}$, and $[A]_{:,j}$ denote the $(i,j)$-th element, the $i$-th row and the $j$-th column of $A$, respectively. We also let $\|A\|_2$ and $\|A\|_F$ denote the spectral norm and the Frobenius norm of $A$, respectively. {\color{black}For a symmetric matrix $A$, we write $A\succ 0$ ($A\succeq 0$, $A\prec 0$, or $A\preceq 0$) when $A$ is positive defnite (positive semi-definite, negative definite,  or negative semi-definite), and $A\succ (\succeq)B$, $A\prec (\preceq)B$ are equivalent to $A-B\succ (\succeq)0$, $A-B\prec (\preceq)0$, respectively.} For a scalar-valued or matrix-valued function of time, $F(t)$, we let $\dot{F}=\dot{F}(t)=\frac{d}{dt}F(t)$ denote its time derivative. Additionally, we let $I_n$ denote the identity matrix of order $n$ and $\mathcal{N}(\mu,\sigma^2)$ denote the normal distribution with mean $\mu$ and variance $\sigma^2$.
\section{Problem Setup}
    \hancheng{This section is mostly the same as icml}
We study the gradient flow on single-hidden-layer linear networks trained with squared $l_2$-loss. Given $N$ training samples $\{x^{(i)},y^{(i)}\}_{i=1}^{N}$, where $x^{(i)}\in\mathbb{R}^n$, $y^{(i)}\in\mathbb{R}^m$, we aim to solve the linear regression problem
    \be
        \min_{\Theta\in\mathbb{R}^{\hl{D}\times m}}\mathcal{L}=\frac{1}{2}\sum_{i=1}^N\|y^{(i)}-\Theta^Tx^{(i)}\|^2_2\,.\label{eq_lin_reg}
    \ee 
    We do so by training a single-hidden-layer linear network $y=f(x;V,U)=VU^Tx$, $V=\mathbb{R}^{m\times h}$, $U\in\mathbb{R}^{n\times h}$, where $h$ is the hidden layer width, with gradient flow, i.e., gradient descent with ``infinitesimal step size". In particular,
    \begin{itemize}
        \item we consider the under-determined case $n>\mathrm{rank}(X)$ for our regression problem, i.e., the input dimension is strictly larger than the rank of $X$. There are infinitely many solutions $\Theta^*$ that achieve optimal loss $\mathcal{L}^*$ of \eqref{eq_lin_reg};
        \item we consider an \emph{overparametrized} model such that $h\geq \min\{m,n\}$, i.e. there is no rank constraint on linear model $\Theta$ obtained from the linear network $UV^T$.
    \end{itemize}
    
    We rewrite the loss with respect to our parameters $V,U$ as
    \be
        \mathcal{L}(V,U)=\frac{1}{2}\sum_{i=1}^N\|y^{(i)}-VU^Tx^{(i)}\|^2_2=\frac{1}{2}\|Y-XUV^T\|^2_F\,,\label{eq_loss_u_v}
    \ee%
    where $Y=[y^{(1)},\cdots,y^{(N)}]^T$ and $X=[x^{(1)},\cdots,x^{(N)}]^T$. 
    The gradient flow dynamics are given by
    \begin{subequations}
    \begin{align}
        \dot{V}(t)=-\frac{\partial\mathcal{L}}{\partial V}(V(t),U(t))=(Y-XU(t)V^T(t))^TXU(t)\,,\label{eq_gf_1}\\ \dot{U}(t)=-\frac{\partial\mathcal{L}}{\partial U}(V(t),U(t))=X^T(Y-XU(t)V^T(t))V(t)\,.\label{eq_gf_2}
    \end{align}
    \end{subequations}
    Since we study the under-determined case, it is necessary to reparametrize the gradient flow dynamics, as shown in the next section. 
    
    (Note: For the rest of this paper, we drop the explicit dependence on time $t$ for scalar/matrix functions of time when such dependence is clear. For example, we will mostly write $U,\dot{U}$ instead of $U(t),\dot{U}(t)$.)

    \subsection{Reparametrization of Gradient Flow}
    Assuming that $n> r=\mathrm{rank}(X)$, the singular value decomposition (SVD) of $X$ can be written as
    \begin{align}
        &\;X=W\bmt \Sigma_x^{1/2} & 0\emt \bmt  \Phi_1^T\\ \Phi_2^T\emt\,,\label{eq_svd_x}
    \end{align}
    where $W\in\mathbb{R}^{N\times r}$, $\Phi_1\in\mathbb{R}^{D\times r}$, and $\Phi_2\in\mathbb{R}^{D\times (D-r)}$. Since $\Phi_1\Phi_1^T+\Phi_2\Phi_2^T=I_D$, we have
    \ben
        U=I_DU=(\Phi_1\Phi_1^T+\Phi_2\Phi_2^T)U=\Phi_1\Phi_1^TU+\Phi_2\Phi_2^TU\,,
    \een
    and hence we can reparametrize $U$ as $(U_1,U_2)$ using the bijection $U=\Phi_1U_1+\Phi_2U_2$, with  inverse $(U_1,U_2)=(\Phi_1^TU,\Phi_2^TU)$. 
    
    We write the gradient flow in \eqref{eq_gf_1}\eqref{eq_gf_2} explicitly as
    \begin{subequations}
    \begin{align}
        \dot{V}&=\;\lp Y-XUV^T\rp^TXU=\;E^T\Sigma_x^{1/2}\Phi_1^TU\,,\label{eq_gf_og_v}\\
        \dot{U}&=\;X^T\lp Y-XUV^T\rp V=\;\Phi_1\Sigma_x^{1/2}E V\,,\label{eq_gf_og_u}
    \end{align}
    \end{subequations}
    where 
    \be E=E(V,U_1):=W^TY-\Sigma_x^{1/2}U_1V^T\,,\label{eq_def_err}
    \ee is defined to be the \emph{error}. Then from \eqref{eq_gf_og_v}\eqref{eq_gf_og_u} we obtain the dynamics in the parameter space $(V,U_1,U_2)$ as
    \be
        \dot{V}=E^T\Sigma_x^{1/2}U_1\,,\ \dot{U}_1=\Sigma_x^{1/2} EV\,,\  \dot{U}_2=0\,.\label{eq_gf_rp}
    \ee
    Notice that
    \begin{align}
        \mathcal{L}(V,U)=\frac{1}{2}\|Y-XUV^T\|^2_F&=\;\frac{1}{2}\|(I-WW^T)Y+WE\|^2_F\nonumber\\
        &=\;\frac{1}{2}\|WE\|^2_F+\frac{1}{2}\|(I-WW^T)Y\|_F^2\nonumber\\
        &=\;\frac{1}{2}\|E\|^2_F+\frac{1}{2}\|(I-WW^T)Y\|_F^2\,,\label{eq_loss_decomp}
    \end{align}
    where the last equality is because $W$ has orthonormal columns. Here the last term in \eqref{eq_loss_decomp} does not dependson $V,U$, and it is the residual
    $$
        \mathcal{L}^*=\frac{1}{2}\|(I-WW^T)Y\|_F^2\,,
    $$
    which is also the optimal value of \eqref{eq_lin_reg}.
    Therefore, for convergence, it suffices to analyze the convergence of the error $E$ under the dynamics of $V,U_1$ in \eqref{eq_gf_rp}. The role of $U_2$ is discussed when we study the implicit bias in Section \ref{sec:gen_lin_net}.
    
    
    \section{Convergence Analysis for Gradient Flow on Single-Hidden-Layer Linear Networks}\label{sec:conv_grad_flow}
    \hancheng{This section is brand-new}
    With the reparametrization of the gradient flow, we study, for convergence, the dynamics of $V,U_1$,
    \ben
        \dot{V}=E^T\Sigma_x^{1/2}U_1\,,\ \dot{U}_1=\Sigma_x^{1/2} EV\,,
    \een
    this is exactly the gradient flow dynamics on\be
        \frac{1}{2}\|E\|_F^2=\frac{1}{2}\|W^TY-\Sigma_x^{1/2}U_1V^T\|_F^2\,.\label{eq_err_frob}
    \ee
    In particular, when $\Sigma_x^{1/2}=I_r$, \eqref{eq_err_frob} reduces to $\frac{1}{2}\|W^TY-U_1V^T\|_F^2$, the loss function for a matrix factorization problem. To motivate our main result, we start with the simplest scalar version of this factorization problem.
    \subsection{Warm-up: Scalar Dynamics}\label{ssec:scalar}
        Consider the gradient flow dynamics on the loss function
        $\mathcal{L}_s(u,v)=\frac{1}{2}|y-uv|^2$, we have
        \be
            \dot{u}=(y-uv)v,\ \dot{v}=(y-uv)u\,.\label{eq_gf_scalar_uv}
        \ee
        This dynamics appear when one studies the gradient flow on \eqref{eq_err_frob} under the spectral initialization~\citep{saxe2014exact,Gidel2019,pmlr-v139-tarmoun21a}. One important feature of \eqref{eq_gf_scalar_uv}, is that the \emph{imbalance} $d:=u^2-v^2$ is invariant under the gradient flow, namely
        \ben
            \dot{d}=2u\dot{u}-2v\dot{v}\equiv 0\,.
        \een
        For the scalar dynamics, such invariance admits explicit solution $u(t),v(t)$ given a fixed imbalance at initialization~\citep{saxe2014exact,pmlr-v139-tarmoun21a}, and the asymptotic convergence rate of $\mathcal{L}_s$ around the equilibrium explicitly depends on the imbalance~\citep{pmlr-v139-tarmoun21a}. In our analysis, the imbalance plays an critical role as well, though in a more global sense. We show two types of initialization that guarantees exponential convergence of $\mathcal{L}_s$: 1) sufficient imbalance; 2) sufficient margin.
        
        One sufficient condition for exponential convergence is a lower bound on the \emph{instantaneous rate} $-\frac{\dot{\mathcal{L}}_s}{\mathcal{L}_s}$, to see this, notice that $\forall t\geq 0$
        \begin{align*}
            -\frac{\dot{\mathcal{L}}_s}{\mathcal{L}_s}\geq c>0 \Ra \int_0^t\frac{\dot{\mathcal{L}}_s(\tau)}{\mathcal{L}_s(\tau)}d\tau \leq \int_0^t-cd\tau \Ra  \log\mathcal{L}_s\biggr\vert^t_0\leq -ct &\Ra\; \log\frac{\mathcal{L}_s(t)}{\mathcal{L}_s(0)}\leq -ct\\
            &\Ra\; \mathcal{L}_s(t)\leq \exp(-ct)\mathcal{L}_s(0)\,,
        \end{align*}
        i.e., a lower bound $c>0$ on the instantaneous rate implies the loss converges to 0 exponentially at a rate at least $c$. Now under the scalar dynamics \eqref{eq_gf_scalar_uv}, one can easily verify that
        \be
            -\frac{\dot{\mathcal{L}}_s}{\mathcal{L}_s}=-\frac{-(y-uv)^2v^2-(y-uv)^2u^2}{(y-uv)^2/2}=2(u^2+v^2)\label{eq_scalar_insta_r_uv}\,.
        \ee
        From the definition of imbalance $d=u^2-v^2$, we have
        \begin{subequations}
        \begin{align}
            u^4&=\;u^2(u^2)=u^2(d+v^2)=du^2+(uv)^2\,,\label{eq_quad_scalar_u}\\
            v^4&=\;v^2(v^2)=v^2(-d+u^2)=-dv^2+(uv)^2\,.\label{eq_quad_scalar_v}
        \end{align}
        \end{subequations}
        Now if we regard the product $uv$ as a known value, then \eqref{eq_quad_scalar_u} and \eqref{eq_quad_scalar_v} are quadratic equations with respect to $u^2$ and $v^2$, whose solutions are
        \be
            u^2=\frac{d+\sqrt{d^2+4(uv)^2}}{2}\,,\ v^2=\frac{-d+\sqrt{d^2+4(uv)^2}}{2}\,.\label{eq_quad_scalar_sol}
        \ee
        Replacing $u^2,v^2$ in \eqref{eq_scalar_insta_r_uv} with the  solutions in \eqref{eq_quad_scalar_sol}, we have
        \be
            -\frac{\dot{\mathcal{L}}_s}{\mathcal{L}_s}=2(u^2+v^2)=2\sqrt{d^2+4(uv)^2}\,,\label{eq_scalar_insta_r_dp}
        \ee
        i.e. the instantaneous rate can be explicitly written as the function of the imbalance $d$ and the product $uv$. More importantly, with proper initialization, we can control the value of $d$ and $uv$ throughout the entire trajectory. Specifically,
        \begin{itemize}
            \item Since the imbalance $d$ is time-invariant, we have $d(t)=d(0)$. When $|d(0)|>0$, there is \emph{sufficient imbalance} at initialization, and
            \ben
                -\frac{\dot{\mathcal{L}}_s(t)}{\mathcal{L}_s(t)}=2\sqrt{d^2(t)+4(u(t)v(t))^2}\geq 2|d(t)|=2|d(0)|\,.
            \een
            \item The product is tied to the loss function $\mathcal{L}_s=|y-uv|^2/2$, and the loss is non-decreasing. When $|y|-|y-u(0)v(0)|>0$, there is \emph{sufficient margin} at initialization, and from 
            \ben
                |u(t)v(t)|\geq |y|-|y-u(t)v(t)|\geq |y|-|y-u(0)v(0)|\,,
            \een
            we have 
            \ben
                -\frac{\dot{\mathcal{L}}_s(t)}{\mathcal{L}_s(t)}=2\sqrt{d^2(t)+4(u(t)v(t))^2}\geq 4|u(t)v(t)|=4(|y|-|y-u(0)v(0)|)\,.
            \een
        \end{itemize}
        Combining the two observations above, we have
        \be
             -\frac{\dot{\mathcal{L}}_s}{\mathcal{L}_s}=2\sqrt{d^2+4(uv)^2}\geq 2\sqrt{d^2(0)+4(\max\{|y|-|y-u(0)v(0)|,0\})^2}\,.\label{eq_scalar_insta_r_lb}
        \ee
        That is, $\mathcal{L}_s$ converges to zero exponentially when either $|d(0)|>0$ (sufficient imbalance) or $|y|-|y-u(0)v(0)|>0$ (sufficient margin). Our main results in the next section show that such observation can be completely generalized to the matrix factorization problem, allowing us to derive exponential convergence guarantees for gradient flow on single-hidden-layer linear networks.
        
    \subsection{Main results}
    
    Now we turn to study the gradient dynamics in \eqref{eq_gf_rp}. Similar to the scalar dynamics, we define the \emph{imbalance} of the single-hidden-layer linear network under input data $X$ as
    \be
        \textit{Imbalance}:\ D=U_1^TU_1-V^TV \in \mathbb{R}^{h\times h}\,.
    \ee
    This imbalance matrix, as expected, is time-invariant under gradient flow dynamics \eqref{eq_gf_rp}. To see this, we compute the time derivative of $U_1^TU_1$ and $V^TV$ as
    \begin{align*}
        \frac{d}{dt}U_1^T  U_1&=\;\dot{U}_1^TU_1+U_1^T\dot{U}_1=V^TE^T\Sigma_x^{1/2}U_1+U_1^T\Sigma_x^{1/2}EV,\\
        \frac{d}{dt}V^T V &=\;  V^T\dot{V}+\dot{V}^TV=V^TE^T\Sigma_x^{1/2}U_1+U_1^T\Sigma_x^{1/2}EV\,.
    \end{align*}
    Because
    $\frac{d}{dt}U_1^TU_1$ and $\frac{d}{dt}V^TV$ are identical, one have $\dot{D}=\frac{d}{dt}[U_1^TU_1-V^TV]\equiv 0$.
    
    Our first result is the lower bound on the instantaneous rate (Proof left to Section \ref{ssec:conv_pf}):
    
    \begin{proposition}[Bound on the instantaneous rate]\label{prop_lb_insta_rate}
    Consider the continuous time dynamics in \eqref{eq_gf_rp}. Let $\tilde{\mathcal{L}}:=\mathcal{L}-\mathcal{L}^*$ and $D=U_1^TU_1-V^TV$, then we have
    \begin{align}
        -\frac{\dot{\tilde{\mathcal{L}}}}{\tilde{\mathcal{L}}}&\geq\; \lambda_r(\Sigma_x)\lp -\Delta_++\sqrt{(\Delta_++\underline{\Delta})^2+4\sigma^2_{m}(U_1V^T)}\right.\nonumber\\
        &\;\quad\quad\quad\quad\quad\quad\quad \left.-\Delta_-+\sqrt{(\Delta_-+\underline{\Delta})^2+4\sigma^2_{r}(U_1V^T)}\rp\,,\label{eq_prop_lb_insta_rate}
    \end{align}
    where we define
    \begin{align}
        \text{(Positive imbalance spectrum spread)}\ \Delta_+&\;=\max\{\lambda_1(D),0\}-\max\{\lambda_r(D),0\}\,,\label{eq_def_eigs1}\\
        \text{(Negative imbalance spectrum spread)}\ \Delta_-&\;=\max\{\lambda_1(-D),0\}-\max\{\lambda_m(-D),0\}\,, \label{eq_def_eigs2}\\
        \text{(Effective level of imbalance)}\ \underline{\Delta}&\;=\max\{\lambda_r(D),0\}+\max\{\lambda_m(-D),0\}\,.\label{eq_def_eigs3}
    \end{align}
    \end{proposition}
\enrique{I would call the third quantity, imbalance spectral gap.}    
    If we think the imbalance $D$ and the product $U_1V^T$ as two factors that contribute to the instantaneous rate, their ``individual" contributions are (assuming $\lambda_r(\Sigma_x)=1$): 1) When  $D=0$, the lower bound reduces to $2\sigma_{\min}(U_1V^T)$ (when $r\neq m$) or $4\sigma_{\min}(U_1V^T)$ (when $r= m$);  2) When $U_1V^T=0$, the lower bound reduces to $2\underline{\Delta}$. That is, the product contributes to the rate through $\sigma_{\min}(U_1V^T)$, while the imbalance does so through level of imbalance $\underline{\Delta}$. As we see in \eqref{eq_prop_lb_insta_rate}, it is not as straightforward as in \eqref{eq_scalar_insta_r_dp} to combine these two factors since extra terms  $\Delta_+, \Delta_-$ enter the lower bound. 
    
    \begin{remark}
        Although in the scalar case the instantaneous rate can be exactly expressed by imbalance and product, the rate also depends on the target $\tilde{Y}=W^TY$ for the general matrix dynamics. Our lower bound in \eqref{eq_prop_lb_insta_rate} is considered optimal when regarding the target as being adversely chosen to minimize the rate. We refer the reader to Appendix \ref{app_tightness} for detailed discussion.
    \end{remark}
    
    
    For the lower bound in \eqref{eq_prop_lb_insta_rate}, one can verify that
    \begin{align*}
        \sqrt{\underline{\Delta}^2+4\sigma^2_{m}(U_1V^T)}&\geq\; -\Delta_++\sqrt{(\Delta_++\underline{\Delta})^2+4\sigma^2_{m}(U_1V^T)}\geq \underline{\Delta}\,,\\
        \sqrt{\underline{\Delta}^2+4\sigma^2_{r}(U_1V^T)}&\geq\; -\Delta_-+\sqrt{(\Delta_-+\underline{\Delta})^2+4\sigma^2_{r}(U_1V^T)}\geq \underline{\Delta}\,.
    \end{align*}
    The right extreme is obtained when $\Delta_+,\Delta_-\ra \infty$ or the singular values are zero, and the left extreme is obtained when $\Delta_+=\Delta_-=0$. 
    
    Therefore, when $\Delta_-,\Delta_+$ are much larger than $\underline{\Delta}$, the lower bound is approximately $$2\lambda_r(\Sigma_x)\underline{\Delta}.$$ When $\Delta_-,\Delta_+$ are much smaller, the bound is approximately $$\lambda_r(\Sigma_x)\lp\sqrt{\underline{\Delta}^2+4\sigma^2_{m}(U_1V^T)}+\sqrt{\underline{\Delta}^2+4\sigma^2_{r}(U_1V^T)}\rp.$$ The latter becomes $2\lambda_r(\Sigma_x)\sqrt{\underline{\Delta}^2+4\sigma^2_{\min}(U_1V^T)}$ when $r=m$, which takes the similar form as in the scalar case. From the experiments in Section \ref{sec:num_expmt}, we see that under random initialization, networks with small width falls into the first regime and ones with large width falls into the latter, and the loss curves behaves differently in these two regimes.
    
    As we illustrated with the scalar dynamics, the lower bound in Proposition \ref{prop_lb_insta_rate}, which depends explicitly on imbalance and product, is useful because one can control the two factors for the entire trajectory with proper initialization. This allows us to derive exponential convergence guarantees for the gradient flow, as stated in our main theorem next (Proof left to Section \ref{ssec:conv_pf}).
    
    \begin{theorem}[Exponential Convergence Guarantee]\label{thm_conv_lin_net}
    Consider the continuous dynamics in \eqref{eq_gf_rp}. Let $\tilde{Y}:=W^TY$ and  define
    \begin{align}
        c(t)&=\;-\Delta_++\sqrt{(\Delta_++\underline{\Delta})^2+4(\max \{\sigma_{m}(\tilde{Y})-\|\tilde{Y}-\Sigma_x^{1/2}U_1(t)V(t)^T\|_F,0\})^2/\lambda_1(\Sigma_x)}\nonumber\\
        &\;\quad -\Delta_-+\sqrt{(\Delta_-+\underline{\Delta})^2+4(\max \{\sigma_{r}(\tilde{Y})-\|\tilde{Y}-\Sigma_x^{1/2}U_1(t)V(t)^T\|_F,0\})^2/\lambda_1(\Sigma_x)}\,,\label{eq_thm_conv_lin_net}
    \end{align}
    where $\Delta_+, \Delta_-$, and $\underline{\Delta}$ are define as in \eqref{eq_def_eigs1},\eqref{eq_def_eigs2}, and \eqref{eq_def_eigs3}. Then we have
    $$
        (\mathcal{L}(t)-\mathcal{L}^*)\leq \exp\lp -\lambda_r(\Sigma_x)c(0)t\rp (\mathcal{L}(0)-\mathcal{L}^*),\forall t\geq 0\,.
    $$
    That is, if $c(0)>0$, then the loss converges to its global minimum exponentially with a rate at least $\lambda_r(\Sigma_x)c(0)$.
    \end{theorem}
    Theorem \ref{thm_conv_lin_net} unifies several previously discovered sufficient conditions for  exponential convergence of the gradient flow on two-layer linear networks:
    \begin{corollary}[Sufficient level of imbalance \citep{mtvm21icml}]\label{col_conv_imb}
        If $\underline{\Delta}>0$ at initialization, then the loss converges to zero exponentially with a rate at least $2\lambda_r(\Sigma_x)\underline{\Delta}(0)$.
    \end{corollary}
    \begin{proof}
        In \eqref{eq_thm_conv_lin_net}, if we lower bound the margin term ($\max \{\sigma_{m}(\tilde{Y})-\|\tilde{Y}-\Sigma_x^{1/2}U_1V^T\|_F,0\}$) by $0$, we have $c\geq 2\underline{\Delta}$.
    \end{proof}
    Previous work~\citep{mtvm21icml} identifies the role of \emph{effective level of imbalance} $\underline{\Delta}$ and proves the convergence result in Corollary \ref{col_conv_imb}. Our result generalizes it by showing the combined contribution of level of imbalance and the margin to the convergence. 
    
    \begin{corollary}[Sufficient margin]\label{col_exp_conv_margin}
        If at initialization, $\sigma_{\min}(\tilde{Y})-\|\tilde{Y}-\Sigma_x^{1/2}U_1V^T\|_F>0$, then $c(0)>0$ and the loss converges to zero exponentially with a rate at least $\lambda_r(\Sigma_x)c(0)$.
    \end{corollary}
    Previous work~\citep{arora2018convergence} has shown that when the initialization has a positive margin, i.e., $\sigma_{\min}(\tilde{Y})-\|\tilde{Y}-\Sigma_x^{1/2}U_1V^T\|_F>0$ and the imbalance has sufficiently small Frobenius norm (approximately balanced), then the gradient flow converges exponentially. Corollary \ref{col_exp_conv_margin} improves upon it by showing that a positive margin is sufficient, regardless of the imbalance.
    
    \begin{corollary}[Characterizing local convergence rate]
        If at some $t_0>0$, we have $c(t_0)>0$, then 
        $$
            (\mathcal{L}(t)-\mathcal{L}^*)\leq \exp\lp -\lambda_r(\Sigma_x)c(t_0)t\rp (\mathcal{L}(t_0)-\mathcal{L}^*),\forall t\geq t_0\,.
        $$
        That is, after $t_0$, the loss converges to zero exponentially with a rate of at least $\lambda_r(\Sigma_x)c(t_0)$. Notably, given any trajectory that eventually converges to a global minimum for $\mathcal{L}$, for sufficiently large $t_0$, we have
        \begin{align}
            c(t_0)&\;\simeq -\Delta_++\sqrt{(\Delta_++\underline{\Delta})^2+4\sigma^2_{m}(\tilde{Y})/\lambda_1(\Sigma_x)} -\Delta_-+\sqrt{(\Delta_-+\underline{\Delta})^2+4\sigma^2_{r}(\tilde{Y})/\lambda_1(\Sigma_x)}\,.\label{eq_col_local_conv}
        \end{align}
    \end{corollary}
    For any trajectory that eventually converges, \eqref{eq_col_local_conv} is due to the fact that $$\|W^TY-\Sigma_x^{1/2}U_1(t_0)V^T(t_0)\|_F\simeq 0\,,$$
    at sufficiently large $t_0$. This corollary suggests that the asymptotic convergence rate around the equilibrium depends on the imbalance $D$ and the target $Y$. Previous work~\citep{pmlr-v139-tarmoun21a} has shown that when $\Sigma_x=I_r$, $h=r=m$ and $D=\lambda I_h$ for some $\lambda\neq 0$, the asymptotic convergence rate of the gradient flow is lower bounded by $2\sqrt{\lambda^2+4\sigma^2_{\min} (\tilde{Y})}$, and this can be exactly recovered from \eqref{eq_col_local_conv} with $\Delta_+=\Delta_-=0,\ \underline{\Delta}=\lambda$. Our result has no additional assumption on the dimension nor on the imbalance structure.
    
    The major limitation of previous works on convergence is the requirement on the imbalance structure: exactly balanced~\citep{arora2018optimization}, or homogeneously imbalanced~\citep{pmlr-v139-tarmoun21a} initialization admits explicit dynamics of the product $U_1V^T$ (the end-to-end function), from which the convergence results are derived. Such analyses considers, as illustrated in Figure \ref{fig_lin_conv_init}, specific configurations in the parameter space, and only allow small variations~\citep{arora2018convergence}. Our analysis breaks such limitation by revealing fundamental relations between the convergence and the weight configuration (imbalance and product), as explicitly seen in the scalar dynamics, which provides convergence guarantees for a wide range of initialization.
    
    \begin{figure}[!h]
        \centering
        \includegraphics[height=9cm]{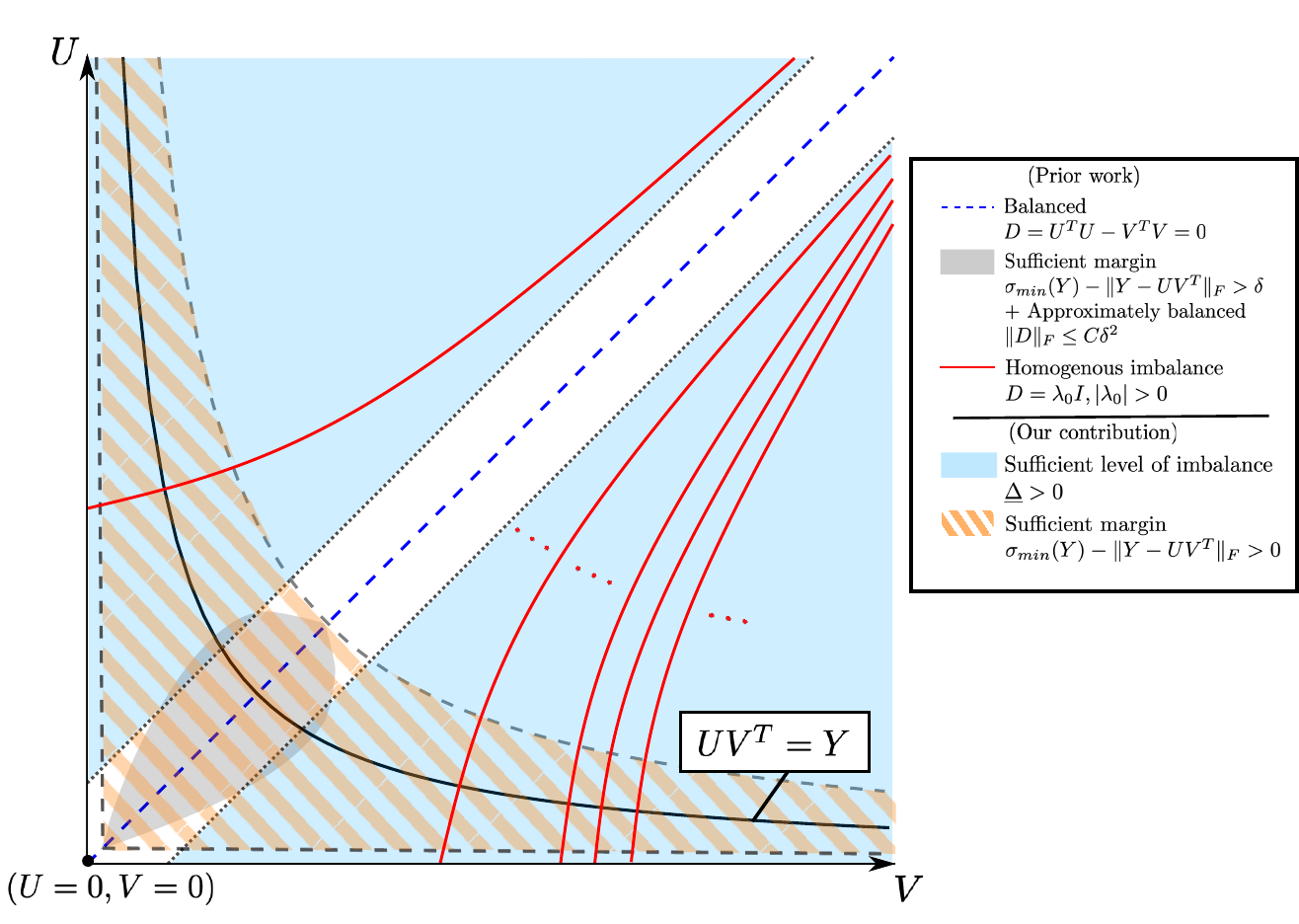}
        \caption{Illustration of non-spectral initialization studied for convergence of linear networks. Note: the conditions are presented for the gradient flow on $\frac{1}{2}\|Y-UV^T\|$, which is the special case of ours when $X=I_n$.}
        \label{fig_lin_conv_init}
    \end{figure}
    
    \subsection{Proof Sketch for the Main Results}\label{ssec:conv_pf}
    
    The proof of our main results Proposition \ref{prop_lb_insta_rate} and Theorem \ref{thm_conv_lin_net} follows exactly the same procedure for the scalar dynamics in Section \ref{ssec:scalar}. We sketch the proof in this section and leave the proofs for all the stated Lemmas to Appendix \ref{app_pf_conv_lem}.
    
    First of all, we lower bound the instantaneous rate with singular values of $U_1,V$, similar to \eqref{eq_scalar_insta_r_uv}.     
    \begin{lemma}\label{lem_lb_u_v}
        Consider the continuous dynamics in \eqref{eq_gf_rp}. Let $\tilde{\mathcal{L}}:=\mathcal{L}-\mathcal{L}^*$, then we have
        $$
        -\frac{\dot{\tilde{\mathcal{L}}}}{\tilde{\mathcal{L}}}\geq 2\lambda_r(\Sigma_x)(\lambda_r(U_1U_1^T)+\lambda_m(VV^T))\,.
        $$
    \end{lemma}
    Recall that for the scalar case, \eqref{eq_quad_scalar_u}\eqref{eq_quad_scalar_v} can be viewed as quadratic equations of $u^2$ and $v^2$ respectively. For the matrix case, one can derive quadratic inequalities of $\lambda_r(U_1U_1^T)$ and of $\lambda_m(VV^T)$, whose solutions give us lower bounds on $\lambda_r(U_1U_1^T)$ and $\lambda_m(VV^T))$, respectively. More generally, we have
    \begin{lemma}\label{lem_lb_a_b}
        Suppose $h\geq \min\{r,m\}$. Given any $A\in\mathbb{R}^{r\times h},B\in\mathbb{R}^{h\times m}$ that satisfy $A^TA-BB^T=D$ for some $D\in\mathbb{R}^{h\times h}$, to get 
        \be
            \lambda_{m}(B^TB)\geq \frac{-\bar{\lambda}+\underline{\lambda}+\sqrt{(\bar{\lambda}+\underline{\lambda})^2+4\sigma_m^2(AB)}}{2}\,, \label{eq_lem_lb_w2}
        \ee
        where $\bar{\lambda}=\max\{\lambda_1(D),0\}$ and $\underline{\lambda}=\max\{\lambda_m(-D),0\}$.
    \end{lemma}
    
    Combining Lemma \ref{lem_lb_u_v} and Lemma \ref{lem_lb_a_b}, we have the desired bound on the instantaneous rate
    \begin{proof}[Proof of Proposition \ref{prop_lb_insta_rate}]
        From Lemma \ref{lem_lb_a_b}, let $A=U_1,B=V^T$, we have $A^TA-BB^T=D$, thus
        \begin{align}
            &\; \lambda_{m}(VV^T)\geq \frac{-\bar{\lambda}_++\underline{\lambda}_-+\sqrt{(\bar{\lambda}_++\underline{\lambda}_-)^2+4\sigma_m^2(U_1V^T)}}{2}\,,\label{eq_pf_prop1_1}\\
            &\;\bar{\lambda}_+=\max\{\lambda_1(D),0\},\ \underline{\lambda}_-=\max\{\lambda_m(-D),0\}\,,\nonumber
        \end{align}
        then let $A=V,B=U_1^T$, we have $A^TA-BB^T=-D$, thus
        \begin{align}
            &\;\lambda_{r}(U_1U_1^T)\geq \frac{-\bar{\lambda}_-+\underline{\lambda}_++\sqrt{(\bar{\lambda}_-+\underline{\lambda}_+)^2+4\sigma^2_r(VU_1^T)}}{2}\,.\label{eq_pf_prop1_2}\\
            &\; \bar{\lambda}_-=\max\{\lambda_1(-D),0\},\ \underline{\lambda}_+=\max\{\lambda_r(D),0\}\nonumber
        \end{align}
        Now rewrite the lowerbounds \eqref{eq_pf_prop1_1}\eqref{eq_pf_prop1_2} in terms of 
        $$
            \Delta_+:=\bar{\lambda}_+-\underline{\lambda}_+,\ \Delta_-:=\bar{\lambda}_--\underline{\lambda}_-,\ \underline{\Delta}:=\underline{\lambda}_++\underline{\lambda}_-\,,
        $$
        we have
        \begin{align*}
            \lambda_{m}(VV^T)\geq \frac{-\bar{\Delta}_++\underline{\lambda}_--\underline{\lambda}_++\sqrt{(\Delta_++\underline{\Delta})^2+4\sigma_m^2(U_1V^T)}}{2}\,,\\
            \lambda_{r}(U_1U_1^T)\geq \frac{-\bar{\Delta}_-+\underline{\lambda}_+-\underline{\lambda}_-+\sqrt{(\Delta_-+\underline{\Delta})^2+4\sigma^2_r(VU_1^T)}}{2}\,.
        \end{align*}
        Then \eqref{eq_prop_lb_insta_rate} follows immediately from Lemma \ref{lem_lb_u_v}.
    \end{proof}
    
    Again, regarding the bound in Proposition \ref{prop_lb_insta_rate}, $\Delta_+, \Delta_-, \underline{\Delta}$ are time-invariant because the imbalance $D$ is so, and the singular value $\sigma_m^2(U_1V^T)$ can be controlled via positive margin. This proves Theorem \ref{thm_conv_lin_net}.
    \begin{proof}[Proof of Theorem \ref{thm_conv_lin_net}]
        When $m=r$, one have $\sigma_m(U_1V^T)=\sigma_r(U_1V^T)=\sigma_{\min}(U_1V^T)$. When $m>r$, we only need to lower bound $\sigma_m(U_1V^T)$ since $\sigma_r(U_1V^T)=0$, and vise versa when $r>m$.
        
        Therefore, without loss of generality, we assume $m\leq r$ and derive the lower bound on $\sigma_m(U_1V^T)$. By $\|A\|_F\geq \|A\|_2$ and Weyl's inequality~\citep[7.3.P16]{Horn:2012:MA:2422911}, one has
        \ben
            \|\tilde{Y}-\Sigma_x^{1/2}U_1V^T\|_F+\sigma_m(\Sigma_x^{1/2}U_1V^T)\geq \|\tilde{Y}-\Sigma_x^{1/2}U_1V^T\|_2+\sigma_m(\Sigma_x^{1/2}U_1V^T)\geq \sigma_m(\tilde{Y})\,,
        \een
        from which one obtain the lower bound
        \ben
            \sigma_m(U_1V^T)\geq \sigma_m(\Sigma_x^{1/2}U_1V^T)/\lambda^{1/2}_1(\Sigma_x)\geq (\sigma_m(\tilde{Y})-\|\tilde{Y}-\Sigma_x^{1/2}U_1V^T\|_F)/\lambda_1^{1/2}(\Sigma_x)\,.
        \een
        The lower bound is trivial when $\sigma_m(\tilde{Y})-\|\tilde{Y}-\Sigma_x^{1/2}U_1V^T\|_F<0$, thus we could write
        \be
            \sigma_m(U_1V^T)\geq \max\{ \sigma_m(\tilde{Y})-\|\tilde{Y}-\Sigma_x^{1/2}U_1V^T\|_F,0\}/\lambda_1^{1/2}(\Sigma_x)\,.
        \ee
        Now because $\|\tilde{Y}-\Sigma_x^{1/2}U_1V^T\|_F=\sqrt{2\tilde{\mathcal{L}}}$ is non-decreasing under gradient flow, we have $\forall t\geq 0$,
        \begin{align}
            \sigma_m^2(U_1(t)V^T(t))&\geq\; (\max\{ \sigma_m(\tilde{Y})-\|\tilde{Y}-\Sigma_x^{1/2}U_1(t)V^T(t)\|_F,0\})^2/\lambda_1(\Sigma_x)\nonumber\\
            &\geq\; (\max\{ \sigma_m(\tilde{Y})-\|\tilde{Y}-\Sigma_x^{1/2}U_1(0)V^T(0)\|_F,0\})^2/\lambda_1(\Sigma_x)\,.\label{eq_thm_conv_pf}
        \end{align}
        Finally using \eqref{eq_thm_conv_pf} to further lower bound \eqref{eq_prop_lb_insta_rate} in Proposition \ref{prop_lb_insta_rate}, we have our desired lower bound on the instantaneous rate
        \ben
            -\frac{\dot{\tilde{\mathcal{L}}}}{\tilde{\mathcal{L}}}\geq \lambda_r(\Sigma_x)c(0)\,.
        \een
        The result $\tilde{\mathcal{L}}(t)\leq \exp(-\lambda_r(\Sigma_x)c(0)t)\tilde{\mathcal{L}}(0)$ follows from Gr\"onwall's inequality~\citep{gronwall1919}.
    \end{proof}
    
    \section{Implicit Bias of Gradient Flow on Single-Hidden-Layer Linear Network}\label{sec:gen_lin_net}
    \hancheng{This section is mostly the same as icml}
    In this section, we study a particular type of implicit bias of single-hidden-layer linear networks under gradient flow. We have assumed that $n>r=\mathrm{rank}(X)$, hence
    the regression problem \eqref{eq_lin_reg} has infinitely many solutions $\Theta^*$  that achieve optimal loss. Among all these solutions, 
    one that is of particular interest in high-dimensional linear regression is
    the \emph{minimum norm solution} (min-norm solution)
    \begin{align}
        \hat{\Theta}&=\;\underset{\Theta\in\mathbb{R}^{n\times m}}{\arg\,\min}\{\|\Theta\|_F:\|Y-X\Theta\|_F^2=\underset{\Theta}{\min}\|Y-X\Theta\|_F^2\}\nonumber\\
        &=\;X^T(XX^T)^\dagger Y ,
    \end{align}
    which has near-optimal generalization error for suitable data models~\citep{bartlett2020benign,mei2019generalization}. Here, we study conditions under which our trained network is equal or close to the min-norm solution by showing how the initialization explicitly controls the trajectory of the training parameters to be exactly (or approximately) confined within some low-dimensional invariant set. In turn,  minimizing the loss over this set leads to the min-norm solution.
    
    
    \subsection{Decomposition of Trained Network}\label{secc:decomp_net_gen}
    Notice that the end-to-end matrix $UV^T\in\mathbb{R}^{D\times m}$ associated with the single-hidden-layer linear network can be decomposed according to the SVD of data matrix $X$, \eqref{eq_svd_x}, as
    \be
        UV^T=(\Phi_1\Phi_1^T+\Phi_2\Phi_2^T)UV^T=\Phi_1U_1V^T+\Phi_2U_2V^T\,,
    \ee
    where $\Phi_1,\Phi_2,U_1,U_2$ are defined in Section~\ref{sec:conv_grad_flow}. The $j$-th column of $UV^T$, $[UV^T]_{:, j}$, is the linear predictor for the $j$-th output $y_j$, and is decomposed into two components within complementary subspaces $\mathrm{span}(\Phi_1)$ and $\mathrm{span}(\Phi_2)$. Moreover $[U_1V^T]_{:, j}$ is the coordinate of $[UV^T]_{:, j}$ w.r.t. the orthonormal basis consisting of the columns of $\Phi_1$, and similarly $[U_2V^T]_{:, j}$ is the coordinate w.r.t. basis $\Phi_2$. Under gradient flow \eqref{eq_gf_rp}, the trajectory $U(t)V(t)^T, t>0$ is fully determined by the trajectory $U_1(t)V^T(t),U_2(t)V^T(t),t>0$.
    
    \noindent
    \textbf{Convergence of Training Parameters}. We have derived useful results regarding $U_1(t)V^T(t)$ for $t>0$ in Section \ref{sec:conv_grad_flow}. When the condition in Theorem \ref{thm_conv_lin_net} is satisfied, exponential convergence of the loss implies $U_1(t)V^T(t)$ converges to some stationary point, as stated in the following proposition
    \begin{proposition}\label{prop_conv_stationary}
        Consider the continuous dynamics in \eqref{eq_gf_rp}. If $c(0)$, defined in Theorem \ref{thm_conv_lin_net}, is positive, then $V(t), U_1(t),t>0$ converges to an equilibrium point $V(\infty),U_1(\infty)$ such that $E(V(\infty),U_1(\infty))=W^TY-\Sigma_x^{1/2}U_1V^T=0$.
    \end{proposition}
    This Proposition is due to the fact that the states in gradient dynamics either converge to an equilibrium point or having its norm grow to infinity, and the exponential convergence excludes the later case. We left its proof to Appendix \ref{app_pf_thm2}. 
    
    Knowing that $V(t),U_1(t)$ converges, it is easy to check that
    \ben
        \Phi_1U_1(\infty)V^T(\infty)=\Phi_1\Sigma_x^{-1/2}W^TY=X^T(XX^T)^\dagger Y=\hat{\Theta}\,.
    \een
    For $U_2(t)V^T(t)$, notice that $\dot{U}_2(t)=0$ in dynamics \eqref{eq_gf_rp}, hence $U_2(t)=U_2(0),\forall t>0$. Overall, given exponential convergence of the loss, $U(t)V^T(t)$ converges to some $U(\infty)V^T(\infty)$ and
    \begin{align}
         U(\infty)V^T(\infty)=\Phi_1U_1(\infty)V^T(\infty)+\Phi_2U_2(0)V^T(\infty)=\hat{\Theta}+\Phi_2U_2(0)V^T(\infty)\,.\label{eq_trained_net_decomp}
    \end{align}
    
    \noindent
    \textbf{Constrained Training via Initialization}. Based on our analysis above, initializing $U_2(0)$ such that $U_2(0)V^T(\infty)=0$ in the limit, guarantees convergence to the min-norm solution via \eqref{eq_trained_net_decomp}. However, this is not easily achievable, as one needs to know a priori $V(\infty)$. Instead, we can show that by choosing a proper initialization, one can constrain the trajectory of the matrix $U(t)V^T(t)$ to lie identically in the set $\Phi_2^TU_2(t)V^T(t)\equiv 0$ for all $t\geq0$, thus the min-norm solution is obtained upon convergence, as suggested by the following proposition.
    \begin{proposition}\label{prop_inv_set}
        Let $V(t),U_1(t),U_2(t), t>0$ be the solution of \eqref{eq_gf_rp} starting from some $V(0),U_1(0),U_2(0)$. We assume $V(t),U_1(t), t>0$ converges to some $V(\infty),U_1(\infty)$ with $E(V(\infty),U_1(\infty))=0$. If the initialization satisfies
        \be
            V(0)U_2^T(0)=0,\ U_1(0)U_2^T(0)=0\,,\label{eq_orth_constraint}
        \ee
        then we have
        $$
            U(\infty)V^T(\infty)=\hat{\Theta}\,.
        $$
    \end{proposition}
    \begin{proof}
        From \eqref{eq_gf_rp} we have
        \be
            \frac{d}{dt}\bmt VU_2^T\\ U_1U_2^T\emt = \bmt 0&  E^T\Sigma^{1/2}_x\\
                \Sigma^{1/2}_xE & 0\emt\bmt VU_2^T\\ U_1U_2^T\emt\,.\label{eq_proj_dym}
        \ee
        Since $VU_2^T=0,\ U_1U_2^T=0$ is an equilibrium point of \eqref{eq_proj_dym}, we have $V(t)U_2^T(0)=0,\forall t\geq 0$ under the initialization in \eqref{eq_orth_constraint}, hence $V(\infty)U_2^T(0)=0$. From \eqref{eq_trained_net_decomp} we conclude that $U(\infty)V^T(\infty)=\hat{\Theta}$.
    \end{proof}
    
    In the standard linear regression, where $\Theta$ follows the gradient flow on $\mathcal{L}(\Theta)=\frac{1}{2}\|Y-X\Theta\|^2_F$, it is well-known that if the columns of $\Theta(0)$ are initialized in $\mathrm{span}(\Phi_1)$, namely $\Theta^T(0)\Phi_2=0$, then $\Theta(\infty)=\hat{\Theta}$. Proposition \ref{prop_inv_set} is the extension of such results to the overparameterized setting. It is worth-noting that initializing the columns of $U(0)V^T(0)$ in $\mathrm{span}(\Phi_1)$, namely $V(0)U_2^T(0)=0$ is no longer sufficient for obtaining $\hat{\Theta}$ as the trained network, and additional condition $U_1(0)U_2^T(0)=0$ is required. 
    
    Here the orthogonality constraints \eqref{eq_orth_constraint} defines an invariant subset of the parameter space $\{V,U: VU_2^T=0,U_1U_2^T=0\}$ under the gradient flow. Proposition \ref{prop_inv_set} shows that given an initialization within the invariant set, the trained network (after convergence) is exactly the min-norm solution, which is the only minimizer in the invariant set.
    
    While in practice we can make the initialization exactly as above, such choice is data-dependent and requires the SVD of the data matrix $X$. Moreover, we note that while the zero initialization works for the standard linear regression case, such initialization $V(0)=0,U(0)=0$ is bad in the overparametrized case because it is an equilibrium point of the gradient flow, even though it satisfies the orthogonal condition $V(0)U_2^T(0)=0$ and $U_1(0)U_2^T(0)=0$.
    
    In the next section, we show that under (properly scaled) random initialization and sufficiently large hidden layer width $h$, both conditions for convergence and implicit bias on initialization are probably approximately satisfied, i.e., with high probability the level of imbalance is sufficient for exponential convergence, and the parameters are initialized close to the invariant set, allowing us to obtain a non-asymptotic bound between the trained network and the min-norm solution.
    
    
    \subsection{Wide Single-Hidden-Layer Linear Network}\label{secc:wide_linear_net}
    In this section, we show how the previously mentioned conditions for convergence and implicit bias, i.e., high imbalance and orthogonality, are approximately satisfied with high probability under the following initialization 
    \begin{align*}
        &\;[U(0)]_{ij}\sim \mathcal{N}\lp 0,\frac{1}{h^{2\alpha}}\rp,\ 1\leq i\leq n, 1\leq j\leq h\,,\\
        &\;[V(0)]_{ij}\sim \mathcal{N}\lp 0,\frac{1}{h^{2\alpha}}\rp,\ 1\leq i\leq m, 1\leq j\leq h\,,
    \end{align*}
    where all the entries are independent and $1/4\leq \alpha\leq 1/2$. 
    
    Both our parametrization and initialization are, at first sight, different from the one used in previous works~\citep{jacot2018neural,du2019width,arora2019exact} on NTK analysis for wide neural networks. We note that with time-rescaling, however, we can relate our initialization to the one in~\citet{arora2019exact}. Please see Appendix \ref{app_comp_ntk} for a comparison.

    Recall form the last section, one can obtain exactly min-norm solution via proper initialization of the single-hidden-layer network. In particular, it requires 1) convergence of the error $E$ to zero; and 2) the orthogonality conditions $V(0)U_2^T(0)=0$ and $U_1(0)U_2^T(0)=0$. Under random initialization and sufficiently large hidden layer width $h$, these two conditions are approximately satisfied. Using basic random matrix theory, one can show the following lemma. See Appendix \ref{app_pf_thm2} for the proof.
    \begin{lemma}\label{lem_si_no} Let $\frac{1}{4}<\alpha\leq \frac{1}{2}$. Given data matrix $X$. $\forall \delta\in(0,1)$, $\forall h>h_0=poly\lp m,n,\frac{1}{\delta}\rp$, with probability at least $1-\delta$ over random initialization with $[U(0)]_{ij},[V(0)]_{ij}\sim\mathcal{N}(0,h^{-2\alpha})$, the following conditions hold:
    \begin{enumerate}[leftmargin=4mm]
        \item (Sufficient level of imbalance)
        \begin{align}
            \underline{\Delta}(0)>h^{1-2\alpha}\,,\label{lem_si_no_eq1}
        \end{align}
        where $\underline{\Delta}$ is the effective level of imbalance defined in \eqref{eq_def_eigs3}.
        \item (Approximate orthogonality)
        \be
        \lV\bmt V(0)U_2^T(0)\\
        U_1(0)U_2^T(0)\emt\rV_F\leq  2\sqrt{m+r}\frac{\sqrt{m+n}+\frac{1}{2}\log \frac{2}{\delta}}{h^{2\alpha-\frac{1}{2}}}\,,\label{lem_si_no_eq2}
        \ee
        \be
        \lV U_1(0)V^T(0)\rV_F\leq 2 \sqrt{m}\frac{\sqrt{m+n}+\frac{1}{2}\log \frac{2}{\delta}}{h^{2\alpha-\frac{1}{2}}}\,.\label{lem_si_no_eq3}
        \ee
    \end{enumerate}
    \end{lemma}
    From \eqref{lem_si_no_eq2}, we know that the parameters are initialized close to the invariant set of our interest, as measured by $\|VU_2^T\|_F+\|U_1U_2^T\|_F$. The dynamics \eqref{eq_proj_dym} quantify at time $t$ how fast this measure can maximally increase given that its current value is non-zero. It is clear that the smaller norm the current error $E$ has, the lower is the rate at which this measure could increase. This suggests that as long as the error converges sufficiently fast, $\|VU_2^T\|_F+\|U_1U_2^T\|_F$ will not increase too much from its initial value. For our purpose, as the width $h$ increases, we need at least a constant rate of exponential convergence of the error (given by \eqref{lem_si_no_eq1}), and an initial error $E(0)$ that is bounded by some constant (derived from \eqref{lem_si_no_eq3}). With these conditions satisfied with high probability, we have the following Theorem regarding the implicit bias of wide linear networks.  We left its proof to Appendix \ref{app_pf_thm2}. 

    \begin{theorem}\label{thm_asymp_conv_min_norm}
        Let $\frac{1}{4}<\alpha\leq \frac{1}{2}$. Let $V(t),U(t), t>0$ be the trajectory of the continuous dynamics \eqref{eq_gf_rp} starting from some $V(0), U(0)$. Then, $\exists C>0$, such that $\forall \delta\in(0,1), \forall h>h_0^{1/(4\alpha-1)}$ with $h_0=poly\lp m,n,\frac{1}{\delta},\frac{\lambda_1(\Sigma_x)}{\lambda_r^3(\Sigma_x)}\rp$, with probability $1-\delta$ over random initializations  with $[U(0)]_{ij},[V(0)]_{ij}\sim\mathcal{N}(0,h^{-2\alpha})$, we have
        \be
            \|U(\infty)V^T(\infty)-\hat{\Theta}\|_2\leq 2C^{1/h^{1-2\alpha}}\sqrt{m+r}\frac{\sqrt{m+n}+\frac{1}{2}\log \frac{2}{\delta}}{h^{2\alpha-\frac{1}{2}}}\,.
        \ee
        Here $C=\exp\lp 1+ \frac{\lambda_1^{1/2}(\Sigma_x)}{\lambda_r(\Sigma_x)}\|Y\|_F\rp$, which depends on the data $X,Y$.
    \end{theorem}
    Previous works~\citep{arora2019exact} show non-asymptotic results on bounding the difference of predictions between the trained network and the kernel predictor of the NTK over a finite number of testing point (non-global result) using more general network structure and activation functions. We work on a simpler model, we are able to study it without going through non-asymptotic NTK analysis, which is considerably more complicated than ours. We believe this theorem is a clear illustration of how overparametrization, in particular, in the hidden layer width, together with random initialization  affects the convergence and implicit bias.
    
    Notably, although our initialization is related to the NTK analysis~\citep{jacot2018neural,arora2019exact} and the kernel regime~\citep{chizat2019lazy}, we significantly simplify the non-asymptotic analysis with the exact charaterization of an invariant set tied to the regularized solution. Specifically, our analysis does not rely on approximating the training flow to one in the infinite width limit, or one from the linearized network at initialization. Instead, we have the exact characterization of the properties required to reach min-norm solution and show how such properties are approximately preserved during training. 
        
    \section{Numerical Experiments}\label{sec:num_expmt}
    In this section, we first illustrate how the imbalance quantities $\Delta_+,\Delta_-,\underline{\Delta}$ are obtained from the spectrum of the imbalance matrix, as well as the role of width in shaping the imbalance quantities under random initialization. Then we run gradient descent (with small step size) on linear regression problem to validate our lower bounds for the convergence rate. We also refer readers to Appendix \ref{app_expmt_lin_implicit} for numerical verification of our Theorem \ref{thm_asymp_conv_min_norm} on implicit bias of wide linear networks.
    \subsection{Imbalance Quantities}
    For simplicity, we consider the matrix factorization problem $\mathcal{L}=\frac{1}{2}\|Y-\frac{1}{\sqrt{mh}}UV^T\|_F^2$, $U\in\mathbb{R}^{r\times h},V\in\mathbb{R}^{m\times h}$ under Xavier initialization~\citep{pmlr-v9-glorot10a}. The scaling factor $\frac{1}{\sqrt{mh}}$ ensures that at initialization, the product $UV^T$ keeps the same scale as we vary the hidden layer width $h$. Our convergence results Proposition \ref{prop_lb_insta_rate} and Theorem \ref{thm_conv_lin_net} apply to this case and the imbalance quantities $\Delta_+,\Delta_-,\underline{\Delta}$ are defined from the imbalance matrix $D=U^TU-V^TV$ at initialization.
    
    When $h\geq n+m$, then with probability 1 under random initialization, the imbalance matrix $D$ has $rank(D)=n+m$ and it has $n$ positive eigenvalues and $m$ negative ones. Our experiment sets $n=20,m=5$ and consider the case of $h=30$ (small width) and $h=1000$ (large width). For initialization, we use $[U(0)]_{ij},[V(0)]_{ij}\sim\mathcal{N}(0,1)$.
    
    \begin{figure}[!h]
    \centering
    \subfloat{{\includegraphics[width=8cm]{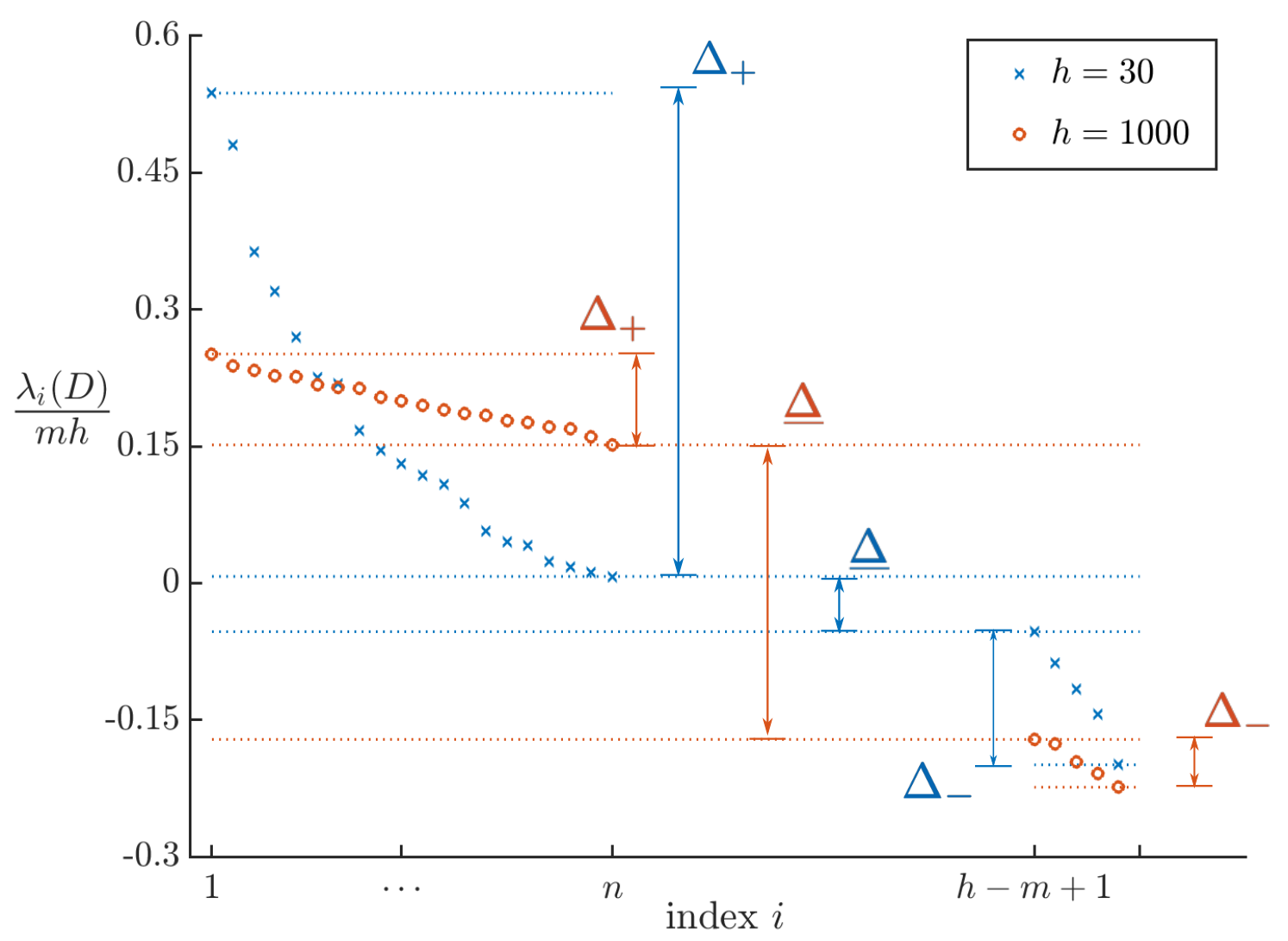} }}%
    \qquad
    \subfloat{{\includegraphics[width=6cm]{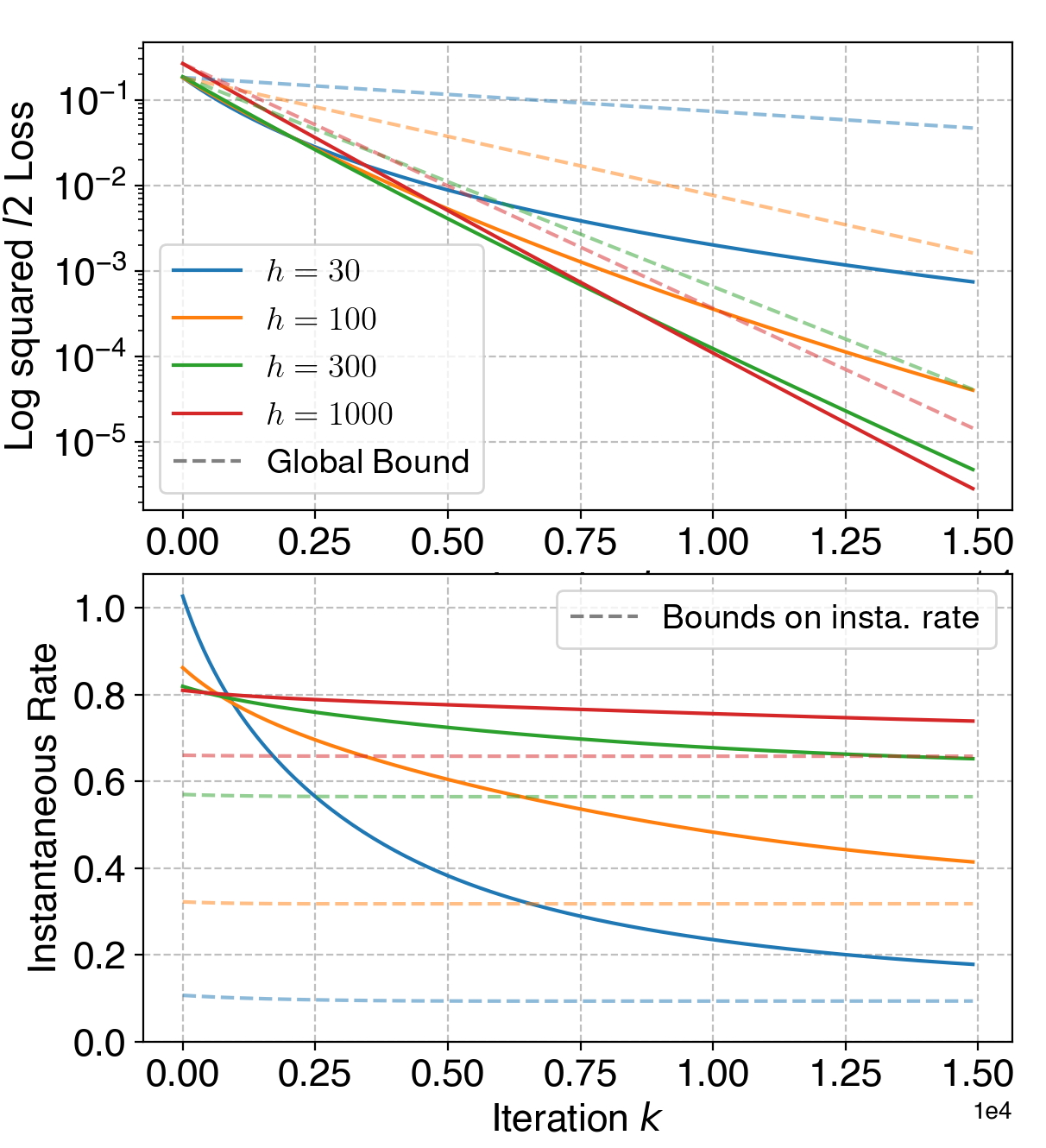} }}%
    \caption{(Left): Scaled eigenvalues of the imbalance matrix $D$ and the corresponding scaled imbalance quantities $\frac{1}{mh}\Delta_+,\frac{1}{mh}\Delta_-,\frac{1}{mh}\underline{\Delta}$ under random initialization, the scaling factor is omitted in the plot annotation for simplicity.\\
    (Right): Gradient descent on $\mathcal{L}=\frac{1}{2}\lV Y-\frac{1}{\sqrt{mh}}UV^T\rV_F^2$ for different network width. The dashed lines represent the bound provided by our results (Proposition \ref{prop_lb_insta_rate} and Theorem \ref{thm_conv_lin_net}).}\label{fig_imb_quant}
    \end{figure}

    Under Xavier initialization, the instantaneous rate is scaled by $\frac{1}{mh}$, hence we consider the scaled imbalance quantities, the details are given in Appendix \ref{app_imb_quant}. We plot in Figure \ref{fig_imb_quant_app} all the non-zero eigenvalues of imbalance $D$ and the imbalance quantities, scaled by $\frac{1}{mh}$. As illustrated by the plot, the imbalance quantities can be understood as the gaps between certain eigenvalues. It is clear that, compare to small width $h=50$, large width $h=1000$ has larger level of imbalance and smaller spectrum spread. 
    
    Moreover, as the width varies, the loss curve behaves differently:

    \noindent
    \emph{(Small width)}: When $h=30$, spectrum spreads $\Delta_-,\Delta_+$ are larger compared to the level of imbalance $\underline{\Delta}$. As we discussed in Section \ref{sec:conv_grad_flow} after Proposition \ref{prop_lb_insta_rate}, the lower bound on the rate is approximately $2\underline{\Delta}$, which is not a good global bound for the convergence rate (see the top plot in Figure \ref{fig_imb_quant_conv_app}). However, interestingly, the instantaneous rate (see the bottom plot in Figure \ref{fig_imb_quant_conv_app}) starts off at large value and decreases as training proceeds. At late stage of the training, our lower bound for the instantaneous rate is reasonably good.
    
    \noindent
    \emph{(Large width)}: When $h=1000$, the level of imbalance $\underline{\Delta}$ is larger compared to spectrum spreads $\Delta_-,\Delta_+$. In this case $2\underline{\Delta}$ is a good global bound on the convergence rate (see the top plot in Figure \ref{fig_imb_quant_conv_app}). As for the instantaneous rate, there is no significant variation in the rate and our bound Proposition \ref{prop_lb_insta_rate} is reasonably good during training.
    
    Such observation hints more complicated relations between imbalance quantities and the training dynamics. We refer the readers to Appendix \ref{app_imb_quant} for more detailed discussion.
    
    \subsection{Convergence via Imbalanced Initialization}
    We train the linear network using gradient descent with a fixed small step size on the averaged loss $\mathcal{L}(U,V)=\|Y-XUV\|^2_F/n$. We use the initialization $U(0)=\sigma_UU_0,V(0)=\sigma_VV_0$ for some randomly sampled $U_0,V_0$ with i.i.d. standard normal entries, and scalars $\sigma_U$, $\sigma_V$. Under this setting, we can change the relative scales of $\sigma_U,\sigma_V$ but keep their product fixed, so that we obtain initializations with different level of imbalance \hl{$c$} while keeping the initial end-to-end matrix $U(0)V^T(0)$ fixed. To eliminate the effect of ill-conditioned $\Sigma_x$ on the convergence, we have $\Sigma_x=I_r$ in this experiment. 

    For comparison, we also consider the balanced initialization that corresponds to the same end-to-end matrix. For a given $\Theta(0)=U(0)V^T(0)$, we choose an arbitrary $Q\in\mathbb{R}^{h\times m}$ with $Q^TQ=I_m$, then a balanced initialization is given by
    \begin{align*}
        U_\mathrm{balanced}(0)&=\;\Theta(0)\lhp \Theta^T(0)\Phi_1\Phi_1^T\Theta(0)\rhp^{-1/4}Q^T,\\
        V_\mathrm{balanced}(0)&=\;\lhp \Theta^T(0)\Phi_1\Phi_1^T\Theta(0)\rhp^{1/4}Q\,.
    \end{align*}
    Such initialization ensures the imbalanced is the zero matrix while keeping the end-to-end matrix as $\Theta(0)$. We note here the choice of $Q$ does not affect the error trajectory $E(t)$, hence the loss $\mathcal{L}(t)$.

    \begin{figure}[ht]
    \centering
    \includegraphics[width=.95\linewidth]{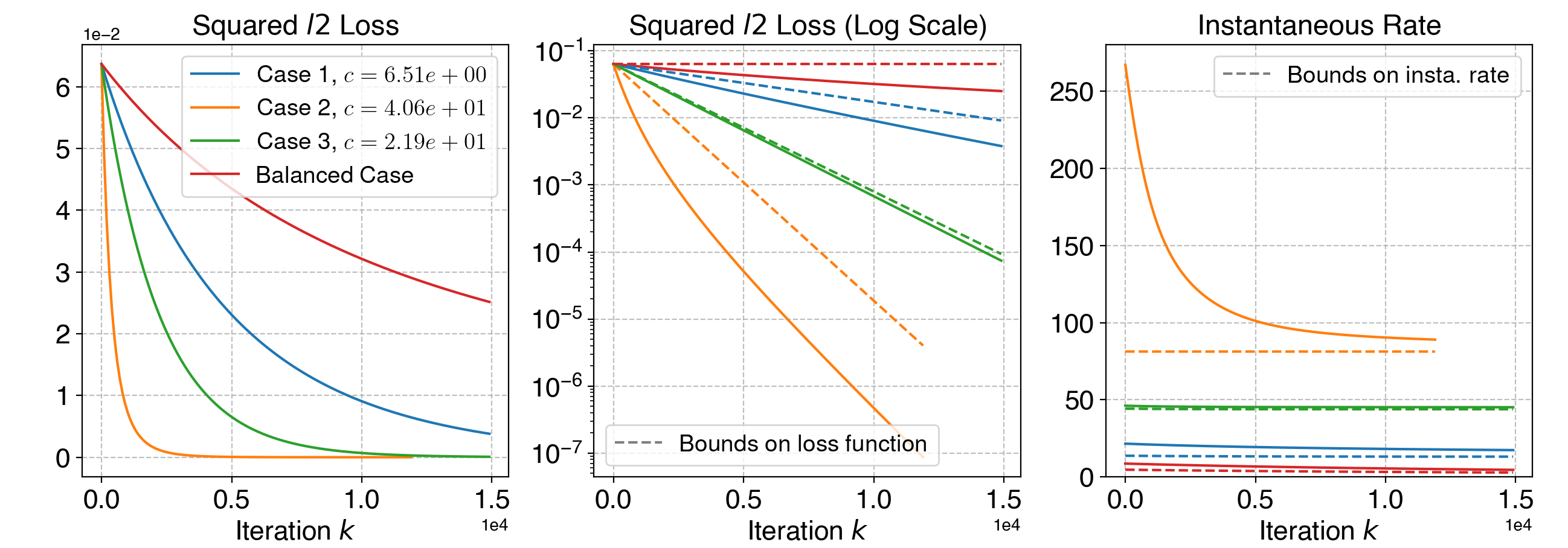}
    \caption{Convergence of gradient descent on linear networks with different initial imbalance matrices. We plot the loss function $\mathcal{L}$ on the left (Regular scale) and the middle(Log scale) figure. The instantaneous rate $-\dot{\mathcal{L}}/\mathcal{L}$ is shown on the right figure. The dashed line on the middle plot shows the bound on loss function by Theorem \ref{thm_conv_lin_net}. Lastly, the dashed line on the right plot shows the lower bound by Proposition \ref{prop_lb_insta_rate}.}
    \label{fig_imbalance_main}
\end{figure}

    From Fig.\ref{fig_imbalance_main}, we see that given fixed step size, the convergence rate is improved as we increase the level of the imbalance at initialization and the balanced initialization is the slowest among all cases. Notably, our lower bound on instantaneous rate is reasonably good for all cases except for case 2 at early training stage. 
    
    Moreover, the randomly initialized end-to-end function $\sigma_U\sigma_VU_0V_0^T$ has zero margin, as there is no bound provided for the balance case (Middle plot in Figure \ref{fig_imbalance_main}). Therefore, the margin-based convergence analysis~\citep{arora2018optimization} relies on carefully chosen initial end-to-end function and fail on the case of random initialization. On the contrary, random initialization almost surely yields a non-zero imbalance matrix, and our bound accounts for the effect of imbalance in convergence, resulting a much tighter bound on the rate.
    
    {\color{black}Note that the goal of this experiment is to verify the improved convergence rate achieved by gradient flow initialized with a high level of imbalance. To this end, we approximate the continuous dynamics using gradient descent with a fixed small step size. However, this does not imply that one can always accelerate gradient descent by increasing the level of imbalance at initialization. This is because the step size for gradient descent is sometimes chosen to be close to the largest possible for convergence, but it is unknown how the level of imbalance affects such choice. Analyzing the  effect of large step size on convergence is subject of current research.}


    \section{Conclusion}
    In this paper, we study the explicit role of initialization on controlling the convergence and implicit bias of single-hidden-layer linear networks trained under gradient flow. We first provide a lower bound on the instantaneous rate based on the imbalance matrix and the product, from which convergence guarantees are derived based on sufficient imbalance or sufficient margin. We then show that proper initialization enforces the trajectory of network parameters to be exactly (or approximately) constrained in a low-dimensional invariant set, over which minimizing the loss yields the min-norm solution. Combining those results, we obtain a novel non-asymptotic bound regarding the implicit bias of wide linear networks under random initialization towards the min-norm solution. 
    Our analysis, although on a simpler overparametrized model,  connects overparametrization, initialization, and optimization. We think it is promising for future research to translate some of the concepts such as the imbalance, and the constrained learning to multi-layer linear networks, and eventually to neural networks with nonlinear activations.
    
    \acks{The authors thank the support of the NSF-Simons Research Collaborations on the Mathematical and Scientific Foundations of Deep Learning (NSF grant 2031985), the NSF HDR TRIPODS Institute for the Foundations of Graph and Deep Learning (NSF grant 1934979), the NSF AMPS Program (NSF grant 1736448), and the NSF CAREER Program (NSF grant 1752362).}

\bibliography{ref.bib}

\appendix
\setcounter{equation}{0}
\def\theequation{\thesection.\arabic{equation}}

\newpage
\section{Numerical Verification}\label{sec:app_num_ver}
\subsection{Effect of imbalance quantities on convergence}\label{app_imb_quant}
\noindent
\textbf{Experiment settings} We consider the matrix factorization problem  $$\mathcal{L}=\frac{1}{2}\lV Y-\frac{1}{\sqrt{mh}}UV^T\rV_F^2,\qquad U\in\mathbb{R}^{r\times h},V\in\mathbb{R}^{m\times h}\,,$$ under Xavier initialization~\citep{pmlr-v9-glorot10a}, where $n=20,m=10$. The weights are initialized as $[U]_{ij},[V]_{ij}\sim\mathcal{N}(0,1)$.
The scaling factor $\frac{1}{\sqrt{mh}}$ ensures that at initialization, the product $UV^T$ keeps the same scale as we vary the hidden layer width $h$. Lastly, we set $[Y]_{ij}\sim\mathcal{N}(0,0.1)$, which gives us a randomly chosen target $Y$ with small norm.

\noindent
\textbf{Imbalance quantities at initialization}
When $h\geq n+m$, then with probability 1 under random initialization, the imbalance matrix $D$ has $rank(D)=n+m$ and it has $n$ positive eigenvalues and $m$ negative ones. Our experiment considers the case of $h=30$ (small width) and $h=1000$ (large width).
\begin{figure}[!h]
    \centering
    \includegraphics[height=7cm]{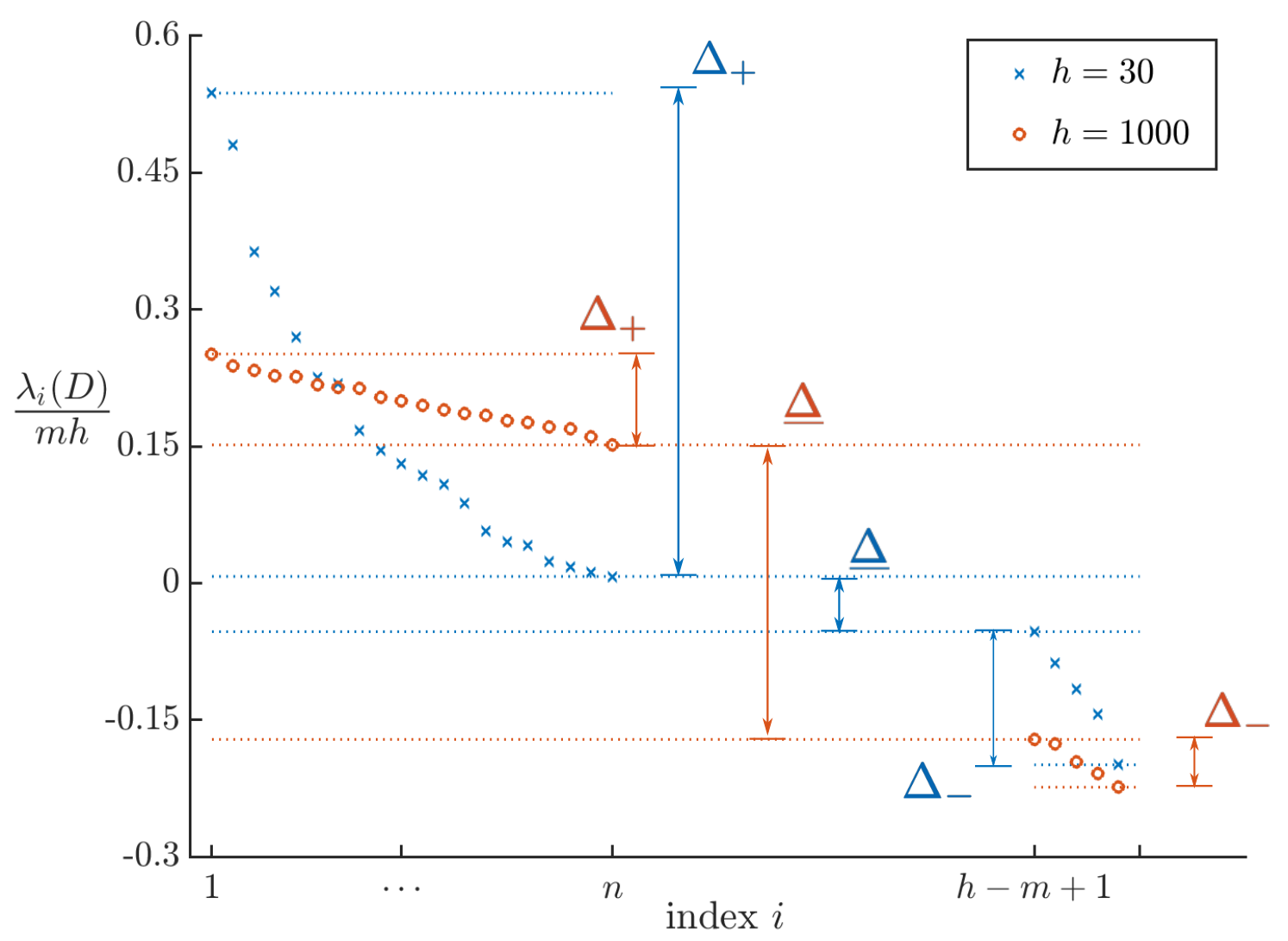}
    \caption{Scaled eigenvalues of the imbalance matrix $D$ and the corresponding scaled imbalance quantities $\frac{1}{mh}\Delta_+,\frac{1}{mh}\Delta_-,\frac{1}{mh}\underline{\Delta}$ under random initialization, the scaling factor is omitted in the plot annotation for simplicity. When the network has small width $h=30$, spectrum spreads $\Delta_+,\Delta_-$ are larger compare to the level of imbalance $\underline{\Delta}$. large width $h=1000$ network shows the opposite.}
    \label{fig_imb_quant_app}
\end{figure}
\begin{remark}
    Notice that the matrix factorization problem with scaling factor $\frac{1}{\sqrt{mh}}$ is equivalent to the regression problem \eqref{eq_loss_u_v} with $X=\frac{1}{\sqrt{mh}}I_n$, the lower bound for the rate is scaled by $\lambda_{\min}(\Sigma_x)=\frac{1}{mh}$. Therefore we analyze the scaled imbalance quantities $\frac{1}{mh}\Delta_+,\frac{1}{mh}\Delta_-,\frac{1}{mh}\underline{\Delta}$.
\end{remark}
We plot in Figure \ref{fig_imb_quant_app} all the non-zero eigenvalues of imbalance $D$ and the imbalance quantities, scaled by $\frac{1}{mh}$. As illustrated by the plot, the imbalance quantities can be understood as the gaps between certain eigenvalues. It is clear that, compare to small width $h=50$, large width $h=1000$ has larger level of imbalance and smaller spectrum spread. 

\noindent
\textbf{Convergence of Gradient Descent} Now under Xavier initialization, we run gradient descent with step size $\eta=0.05$ and plot
\begin{itemize}
    \item The loss function $\mathcal{L}$ in log scale, along with the bound given in Theorem \ref{thm_conv_lin_net};
    \item The instantaneous rate $-\frac{\dot{\mathcal{L}}}{\mathcal{L}}$, along with the bound given in Proposition \ref{prop_lb_insta_rate};
\end{itemize}
for each iteration. We run the experiment under different width $h=30,100,300,1000$.
\begin{figure}[!h]
    \centering
    \includegraphics[height=10cm]{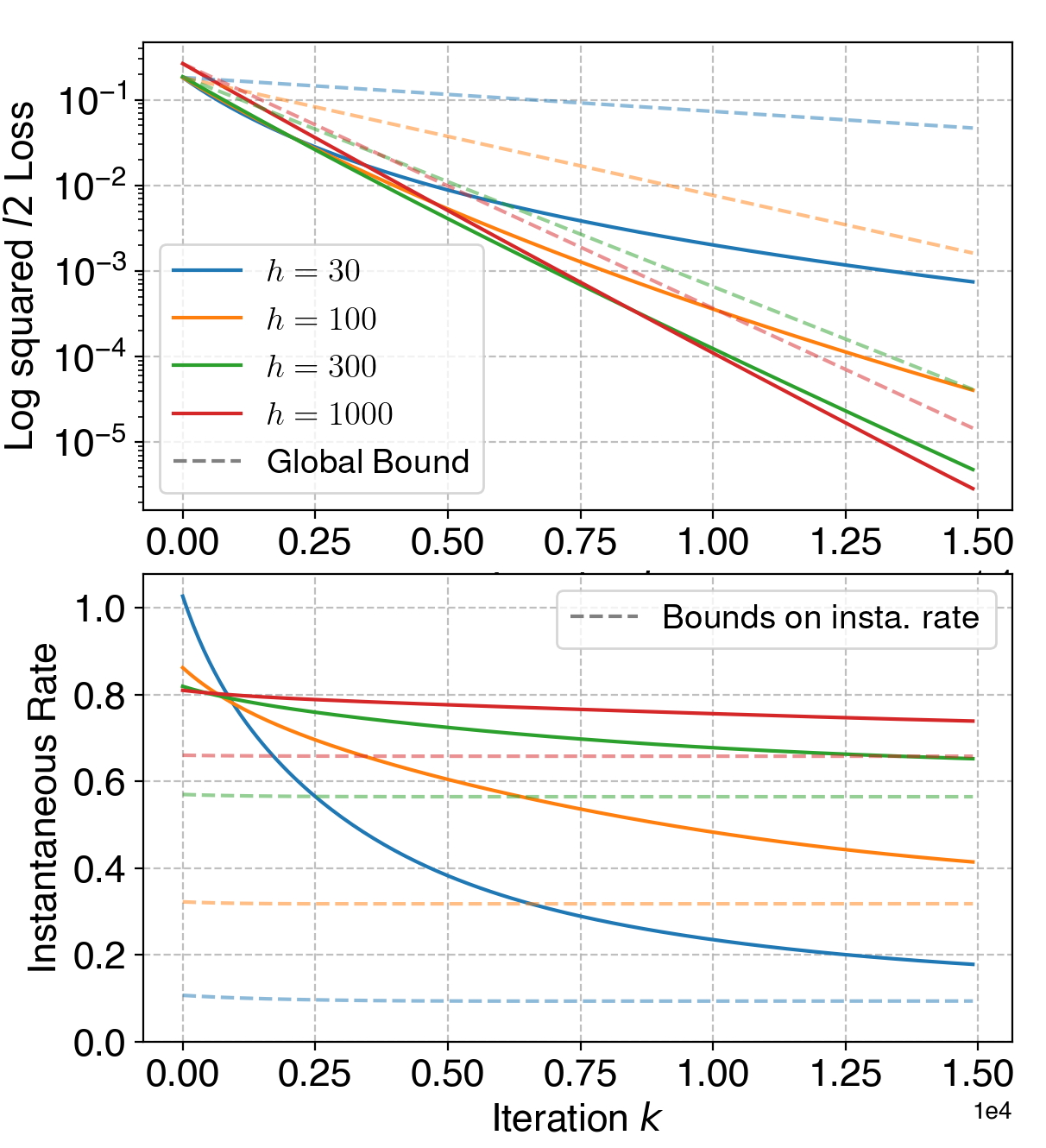}
    \caption{Gradient descent on $\mathcal{L}=\frac{1}{2}\lV Y-\frac{1}{\sqrt{mh}}UV^T\rV_F^2$.}
    \label{fig_imb_quant_conv_app}
\end{figure}

As the width varies, the loss curve behaves differently:

\noindent
\emph{(Small width)}: When $h=30$, spectrum spreads $\Delta_-,\Delta_+$ are larger compared to the level of imbalance $\underline{\Delta}$. As we discussed in Section \ref{sec:conv_grad_flow} after Proposition \ref{prop_lb_insta_rate}, the lower bound on the rate is approximately $2\underline{\Delta}$, which is not a good global bound\footnote{Under random initialization, the margin term in Theorem \ref{thm_conv_lin_net} is zero with high probability. Therefore the global bound generally depends on the imbalance quantities only.} for the convergence rate (see the top plot in Figure \ref{fig_imb_quant_conv_app}). However, interestingly, the instantaneous rate (see the bottom plot in Figure \ref{fig_imb_quant_conv_app}) starts off at large value and decreases as training proceeds. At late stage of the training, our lower bound for the instantaneous rate is reasonably good.

\noindent
\emph{(Large width)}: When $h=1000$, the level of imbalance $\underline{\Delta}$ is larger compared to spectrum spreads $\Delta_-,\Delta_+$. In this case $2\underline{\Delta}$ is a good global bound on the convergence rate (see the top plot in Figure \ref{fig_imb_quant_conv_app}). As for the instantaneous rate, there is no significant variation in the rate and our bound Proposition \ref{prop_lb_insta_rate} is reasonably good during training.

Our analysis provide some insights to these observations: Following the analysis in Appendix \ref{app_tightness}, the dynamics of the error $E=Y-\frac{1}{\sqrt{mh}}UV^T$ can be written as $\dot{E}=\frac{1}{mh}\mathcal{T}_tE$, where $\mathcal{T}_t$ is a time-variant linear operator on $\mathbb{R}^{n\times m}$. Moreover, the eigenvalues of $\mathcal{T}_t$, which characterize the convergence rate of error in different directions, can be explicitly expressed as $\lambda_i(U(t)U(t)^T)+\lambda_j(V(t)V(t)^T),1\leq i\leq n,1\leq j\leq m$. When $UV^T$ has small norm during training, which is the case in our experiment with target $Y$ having small norm, positive eigenvalues of the imbalance serve as a good approximate to $\lambda_i(U(t)U(t)^T),1\leq i\leq n$ and negative eigenvalues serve as a good approximate to $-\lambda_j(V(t)V(t)^T),1\leq i\leq m$. 

When the width is small, there is large spectrum spread for the eigenvalues of the imbalance matrix, which implies the eigenvalues of $\mathcal{T}_t$ have large spread as well. The error $E$ converges faster in some directions but much slower in others, and our lower bound only accounts for the slowest direction in which the error converges. Therefore, the lower bound in Proposition \ref{prop_lb_insta_rate} is not tight at early stage of the training. The bound becomes better as training proceeds because at late stage, the main component of the error lies in the slow directions. On the contrary, when the width is large, small spectrum spread implies that the eigenvalues of $\mathcal{T}_t$ all concentrate at a certain value, and our lower bound accurately characterize the convergence rate of error in every directions.

In summary, for the convergence of linear networks, we observe two regimes, depending on the relative values between the spectrum spread and the level of imbalance, where the loss curve behaves differently. Through our experiment, we show that random initialization could fall into one of the regime depending on the network width. Our analysis hints some relation between the imbalance quantities $\Delta_+,\Delta_-,\underline{\Delta}$ and the behavior of the loss curve, and establishing such connection formally is left to future research.


\subsection{Convergence of single-hidden-layer linear network via imbalanced initialization}\label{app_expmt_lin_conv_imb}
The scale of the linear regression problem we consider in Section \ref{app_expmt_lin_conv_imb} and \ref{app_expmt_lin_implicit} is $D=500$, $n=100$, and $m=1$.

\noindent
\textbf{Generating training data} The synthetic training data is generated as following:

1) For data matrix $X$, first we generate $X_0\in\mathbb{R}^{n\times D}$ with all the entries sampled from $\mathcal{N}(0,1)$, and take its SVD $X_0=W\Sigma^{1/2}\Phi_1$. Then we let $X=W\Phi_1$, hence we have all the singular values of $X$ being 1. Here $r=\mathrm{rank}(X)=n=100$.

2) For $Y$, we first sample $\Theta\sim\mathcal{N}(0,D^{-1}I_D)$, and $\epsilon\sim\mathcal{N}(0,0.01^2I_n)$, then we let $Y=X\Theta+\epsilon$. 

\noindent
\textbf{Initialization and Training} We set the hidden layer width $h=500$. We initialize $U(0),V(0)$ with 
$$
    U(0)=\sigma_U U_0,\ V(0)=\sigma_V V_0,\ [U_0]_{ij},[V_0]_{ij}\overset{i.i.d.}{\sim} \mathcal{N}(0,1)\,,
$$
 and we consider three cases of such initialization: 1) $\sigma_U=0.1,\ \sigma_V=0.1$; 2) $\sigma_U=0.5,\ \sigma_V=0.02$; 3) $\sigma_U=0.05,\ \sigma_V=0.2$. Such setting ensures the initial end-to-end function are identical for all cases but with different imbalance matrices. For these three cases, we run gradient descent on the averaged loss $\tilde{L}=\frac{1}{n}\|Y-XUV^T\|_F^2$ with step size\footnote{To compute the bound from Theorem 1, the step size is scaled by $n/2$ to account for that the gradient descent uses rescaled loss function.} $\eta=5e-4$. 

For comparison, we also consider the balanced initialization that corresponds to the same end-to-end matrix. For a given $\Theta(0)=U(0)V^T(0)$, we choose an arbitrary $Q\in\mathbb{R}^{h\times m}$ with $Q^TQ=I_m$, then a balanced initialization is given by
\ben
    U_\mathrm{balanced}(0)=\Theta(0)\lhp \Theta^T(0)\Phi_1\Phi_1^T\Theta(0)\rhp^{-1/4}Q^T,\quad
    V_\mathrm{balanced}(0)=\lhp \Theta^T(0)\Phi_1\Phi_1^T\Theta(0)\rhp^{1/4}Q\,.
\een
Such initialization ensures the imbalanced is the zero matrix while keeping the end-to-end matrix as $\Theta(0)$. We note here the choice of $Q$ does not affect the error trajectory $E(t)$, hence the loss $\mathcal{L}(t)$.

\begin{figure}[ht]
    \centering
    \includegraphics[height=5.3cm]{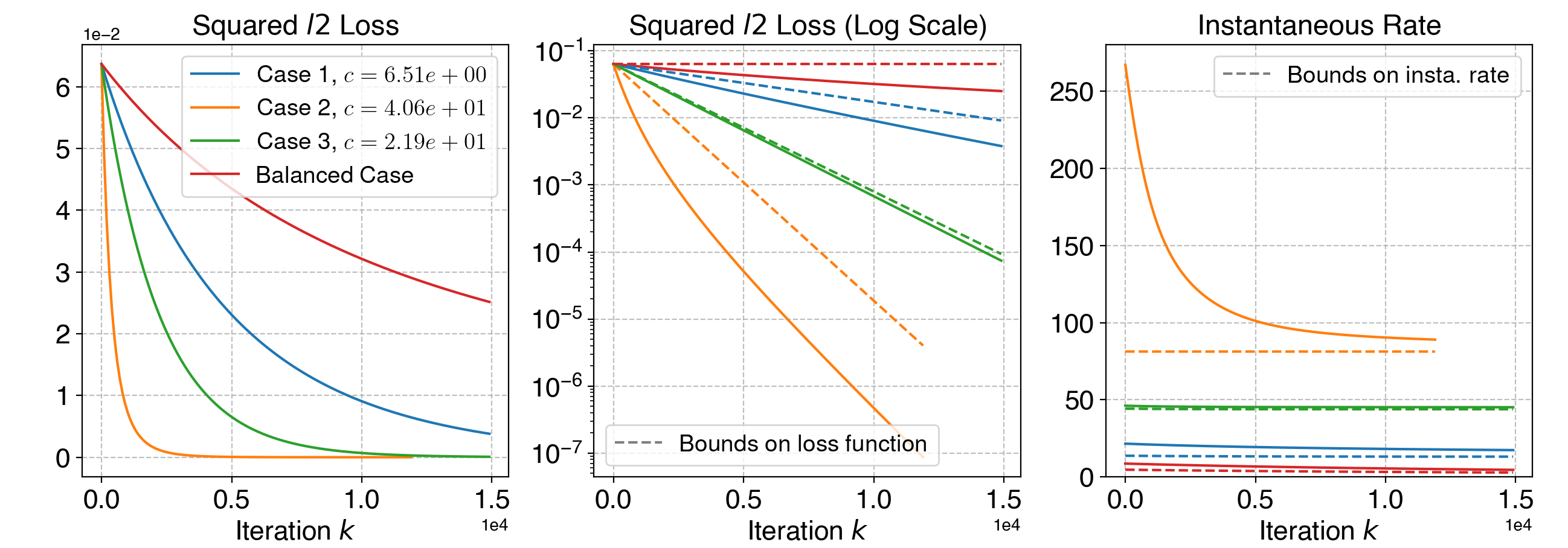}
    \caption{Convergence of gradient descent on linear networks with different initial imbalance matrices. We plot the loss function $\mathcal{L}$ on the left (Regular scale) and the middle(Log scale) figure. The instantaneous rate $-\dot{\mathcal{L}}/\mathcal{L}$ is shown on the right figure. The dashed line on the middle plot shows the bound on loss function by Theorem \ref{thm_conv_lin_net}. Lastly, the dashed line on the right plot shows the lower bound by Proposition \ref{prop_lb_insta_rate}.}
    \label{fig1_imbalance}
\end{figure}

From Fig.\ref{fig1_imbalance}, we see that given fixed step size, the convergence rate is improved as we increase the level of the imbalance at initialization and the balanced initialization is the slowest among all cases. Notably, our lower bound on instantaneous rate is reasonably good for all cases except for case 2 at early training stage. 
    
    Moreover, the randomly initialized end-to-end function $\sigma_U\sigma_VU_0V_0^T$ has zero margin, as there is no bound provided for the balance case (Middle plot in Figure \ref{fig_imbalance_main}). Therefore, the margin-based convergence analysis~\citep{arora2018optimization} relies on carefully chosen initial end-to-end function and fail on the case of random initialization. On the contrary, random initialization almost surely yields a non-zero imbalance matrix, and our bound accounts for the effect of imbalance in convergence, resulting a much tighter bound on the rate.

\subsection{Implicit regularization on wide single-hidden-layer linear network}\label{app_expmt_lin_implicit}
\textbf{Generating training data} The synthetic training data is generated as following:

1) For data matrix $X$, first we generate $X\in\mathbb{R}^{n\times D}$ with all the entries sampled from $\mathcal{N}(0,D^{-1})$;

2) For $Y$, we first sample $\Theta\sim\mathcal{N}(0,D^{-1}I_D)$, and $\epsilon\sim\mathcal{N}(0,0.01^2I_n)$, then we let $Y=X\Theta+\epsilon$. 

\textbf{Initialization and Training} We initialize $U(0),V(0)$ with $[U(0)]_{ij}\sim\mathcal{N}(0,h^{-1})$, $[V(0)]_{ij}\sim\mathcal{N}(0,h^{-1})$ and run gradient descent on the averaged loss $\tilde{L}=\frac{1}{n}\|Y-XUV^T\|_F^2$ with step size $\eta=5e-3$. The training stops when the loss is below $1e-8$. We run the algorithm for various $h$ from $500$ to $10000$, and we repeat 5 runs for each $h$.
\begin{figure}[ht]
    \centering
    \includegraphics[height=4.5cm]{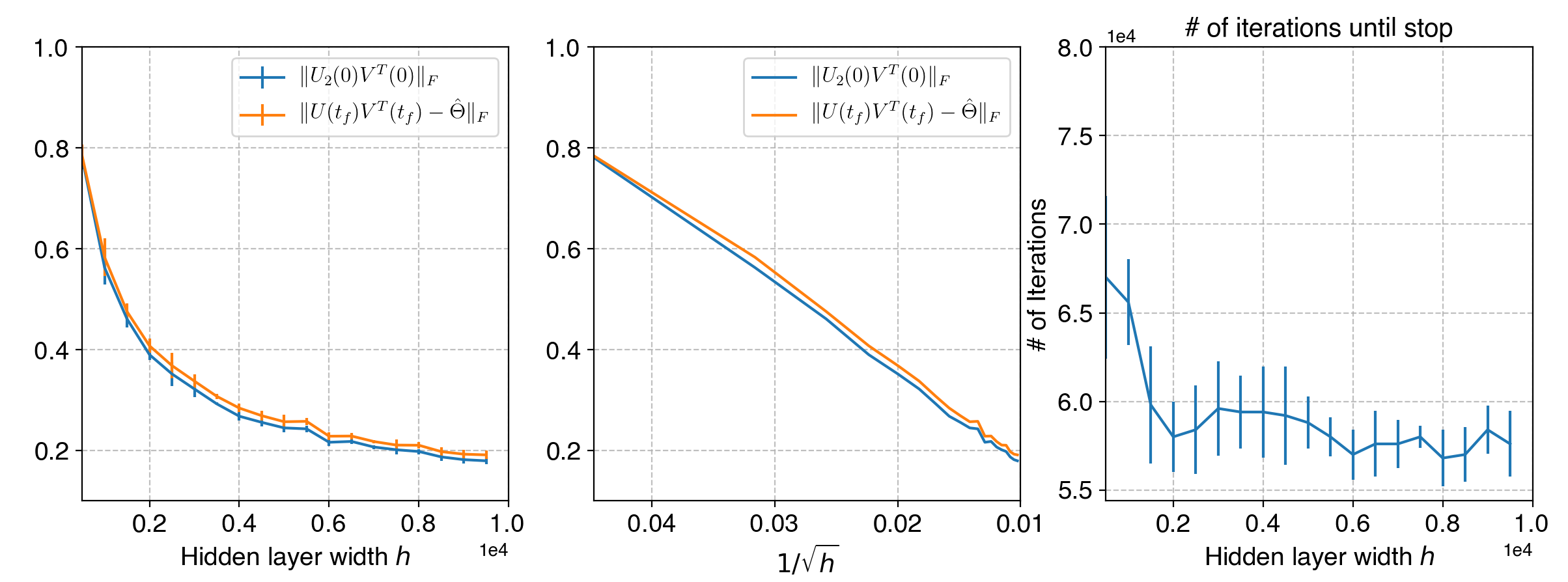}
    \caption{Implicit bias of wide single-hidden-layer linear network under random initialization. The line is plotting the average over 5 runs for each $h$, and the error bar shows the standard deviation. The gradient descent stops at iteration $t_f$.}
    \label{fig2_dist_to_min_norm}
\end{figure}

Fig.\ref{fig2_dist_to_min_norm}  clearly shows that the distance between the trained network and the min-norm solution, $\|U(t_f)V^T(t_f)-\hat{\Theta}\|_F$, decreases as the width $h$ increases and the middle plot verifies the asymptotic rate $\mathcal{O}(h^{-1/2})$.

\section{Proofs of Lemma \ref{lem_lb_u_v} and \ref{lem_lb_a_b}}\label{app_pf_conv_lem}
\begin{proof}[Proof of Lemma \ref{lem_lb_u_v}]
    Under \eqref{eq_gf_rp}, the time derivative of error is given by
    \ben
        \dot{E}=-\Sigma^{1/2}_{x}U_1U_1^T\Sigma^{1/2}_{x}E-\Sigma_xEVV^T\,.
    \een
    Consider the time derivative of $\|E\|_F^2$,
    \be
        \frac{d}{dt}\|E\|_F^2= \frac{d}{dt}\tr(E^TE)=-2\tr\lp E^T\Sigma^{1/2}_{x}U_1U_1^T\Sigma^{1/2}_{x}E+E^T\Sigma_xEVV^T\rp\,.\label{eq_dt_err_fro_norm}
    \ee
    Use the trace inequality~\citep[Lemma 1]{wang1986} to get the lower bound the trace of two matrices respectively as
    \begin{align}
        \tr\lp E^T\Sigma^{1/2}_{x}U_1U_1^T\Sigma^{1/2}_{x}E\rp
        =&\;\tr\lp \Sigma^{1/2}_{x}EE^T\Sigma^{1/2}_{x}U_1U_1^T\rp\nonumber\\
        \geq&\; \lambda_r(U_1U_1^T)\tr\lp \Sigma^{1/2}_{x}EE^T\Sigma^{1/2}_{x}\rp\nonumber\\
        =&\; \lambda_r(U_1U_1^T)\tr\lp \Sigma_{x}EE^T\rp\nonumber\\
        \geq&\; \lambda_r(U_1U_1^T)\lambda_r(\Sigma_x)\tr(EE^T)\nonumber\\
        =&\;\lambda_r(U_1U_1^T)\lambda_r(\Sigma_x)\|E\|_F^2\,,\label{eq_dt_err_fro_norm_lb_u}
    \end{align}
    and
    \begin{align}
        \tr\lp E^T\Sigma_x EVV^T\rp
        \geq&\; \lambda_m(VV^T)\tr\lp E^T\Sigma_x E\rp\nonumber\\
        =&\;\lambda_m(VV^T)\tr\lp \Sigma_x EE^T\rp\nonumber\\
        \geq&\; \lambda_m(VV^T)\lambda_r(\Sigma_x)\tr(EE^T)\nonumber\\
        =&\;\lambda_m(VV^T)\lambda_r(\Sigma_x)\|E\|_F^2\,.\label{eq_dt_err_fro_norm_lb_v}
    \end{align}
    Combine \eqref{eq_dt_err_fro_norm} with \eqref{eq_dt_err_fro_norm_lb_u}\eqref{eq_dt_err_fro_norm_lb_v}, we have
    \be
        \frac{d}{dt}\|E\|_F^2\leq -2\lambda_r(\Sigma_x)\lp \lambda_r(U_1U_1^T)+\lambda_m(VV^T)\rp \|E\|_F^2\label{eq_dt_err_fro_norm_lb}
    \ee
    Notice that $\frac{1}{2}\|E\|_F^2$ is exactly $\tilde{\mathcal{L}}=\mathcal{L}-\mathcal{L}^*$. It follows from \eqref{eq_dt_err_fro_norm_lb} that
    \ben
        -\frac{\dot{\tilde{\mathcal{L}}}}{\tilde{\mathcal{L}}}\geq 2\lambda_r(\Sigma_x)\lp \lambda_r(U_1U_1^T)+\lambda_m(VV^T)\rp\,.
    \een
\end{proof}

\begin{proof}[Proof of Lemma \ref{lem_lb_a_b}]
    
    From the imbalance equation $A^TA-BB^T=D$, we have
    \begin{align*}
        (B^TB)^2= B^T(BB^T)B
        = B^T(A^TA-D)B= B^TA^TAB-B^TDB\,.
    \end{align*}
    Let $z_m\in\mathbb{S}^{m-1}$ be the eigenvector of $(B^TB)^2$ (or $B^TB$) associated with eigenvalue $\lambda_m^2(B^TB)$ (or $\lambda_m(B^TB)$). The one have
    \begin{align}
        \lambda_m^2(B^TB)=z_m^T(B^TB)^2z_m&=\;z_m^TB^TA^TABz_m-z_m^TB^TDBz_m\nonumber\\
        &\geq \; \lambda_m(B^TA^TAB)-z_m^TB^TDBz_m\,,\nonumber\\
        &=\; \sigma_m^2(AB)-z_m^TB^TDBz_m\label{eq_lem_lb_w2_eq1}
    \end{align}
    and the rest of proof is to find a lower bound for $-z^T_mB^TDBz_m$.
    
    First of all, we know that $D$ has at most $m$ negative eigenvalues: If $D$ has more than $m$ negative eigenvalues, then the subspace spanned by the all negative eigenvectors has dimension at least $m+1$, which must have non-trivial intersection with $\ker(B^T)$, then there exists a nonzero vector $z\in\ker(B^T)$ such that $z^TDz<0$, which would imply $z^TA^TAz=z^TDz<0$, a contradiction. 
    
    \underline{When $D$ has less than $m$ negative eigenvalues}, then $\underline{\lambda}=0$ and we simply lower bound $-z^T_mB^TDBz_m$ as
    \begin{align*}
        \lambda_m^2(B^TB)&\geq \; \sigma_m^2(AB)-z_m^TB^TDBz_m\\
        &\geq \; \sigma_m^2(AB)-\bar{\lambda} z_m^TB^TBz_m\\
        &= \;  \sigma_m^2(AB)-\bar{\lambda} \lambda_m(B^TB)\,.
    \end{align*}
    This quadratic inequality w.r.t. $\lambda_m(B^TB)$ has nonnegative solutions
    $$
        \lambda_m(B^TB)\geq\frac{-\bar{\lambda}+\sqrt{\bar{\lambda}^2+4\sigma_m^2(AB)}}{2}\,,
    $$
    which is exactly \eqref{eq_lem_lb_w2} when $\underline{\lambda}=0$.
    
    \underline{When $D$ has exactly $m$ negative eigenvalues}, the easy case is one with $h=m$, i.e. all eigenvalues of $D$ are negative. We simply lower bound $-z^T_mB^TDBz_m$ as
    \begin{align*}
        \lambda_m^2(B^TB)&\geq \; \sigma_m^2(AB)-z_m^TB^TDBz_m\\
        &\geq \; \sigma_m^2(AB)-(-\underline{\lambda} z_m^TB^TBz_m)\\
        &= \;  \sigma_m^2(AB)+\underline{\lambda} \lambda_m(B^TB)\,.
    \end{align*}
    This quadratic inequality w.r.t. $\lambda_m(B^TB)$ has nonnegative solutions
    $$
        \lambda_m(B^TB)\geq\frac{\underline{\lambda}+\sqrt{\underline{\lambda}^2+4\sigma_m^2(AB)}}{2}\,,
    $$
    which is exactly \eqref{eq_lem_lb_w2} when $\bar{\lambda}=0$.

    Now we only left to prove the bound for the case $h>m$. We first consider any orthogonal matrix $Q\in\mathcal{O}(h)$, we have $Q^TA^TAQ-Q^TBB^TQ=Q^TDQ$, $AQQ^TB=AB$, and $\lambda_m(B^TQ^TQB)=\lambda_m(B^TB)$. Then it suffices to study the rotated matrices $\tilde{A}=AQ,\tilde{B}=Q^TB$, with $\tilde{A}^T\tilde{A}-\tilde{B}\tilde{B}^T=Q^TDQ,\tilde{A}\tilde{B}=AB$ and find a lower bound on $\lambda_m(\tilde{B}^T\tilde{B})$. We can pick $Q$ that diagonalize $D$, thus \emph{with out loss of generality, we assume $D$ is diagonal and the eigenvalues are in decreasing order}.
    
    Since $h>m$, we write the diagonal $D$ as a block matrix
    $
    D= \bmt \Lambda_+&0\\
    0&-\Lambda_-\emt\,,$
    where
    \begin{align*}
        &\;\Lambda_+=\mathrm{diag}\{\lambda_1(D),\cdots,\lambda_{h-m}(D)\}\\
        &\; \Lambda_-=\mathrm{diag}\{-\lambda_{h-m+1}(D),\cdots,-\lambda_{h}(D)\}=\mathrm{diag}\{\lambda_{m}(-D),\cdots,\lambda_{1}(-D)\}\,.
    \end{align*} 
    Here, notice that $\Lambda_+$ is positive semi-definite and $\Lambda_-$ positive definite with 
    \be
        \Lambda_+\preceq \bar{\lambda} I_{h-m},\ \Lambda_-\succeq \underline{\lambda} I_m\,.\label{eq_blk_lambda_pd}
    \ee
    Now we write $A,B$ as block matrices as well
    \begin{align*}
        &\;A=\bmt A_+&A_-\emt,\ B=\bmt B_+\\B_-\emt,\\ &\;A_+\in\mathbb{R}^{r\times (h-m)}, A_-\in\mathbb{R}^{r\times m}, B_+\in\mathbb{R}^{(h-m)\times m},B_-\in\mathbb{R}^{m\times m}\,, 
    \end{align*}
    from which we can rewrite equations $A^TA-BB^T=D$ as
    \begin{align*}
        \bmt A_+^T\\ A_-^T
        \emt\bmt A_+&A_-\emt-\bmt B_+\\ B_-
        \emt\bmt B_+^T&B_-^T\emt&=\;\bmt \Lambda_+&0\\
        0&-\Lambda_-\emt\,.
    \end{align*}
    By inspection, the equality for each block gives us
    \begin{align}
        A_+^TA_+&=\;B_+B_+^T+\Lambda_+\,,\label{eq_imb_blk1}\\
        A_-^TA_-&=\;B_-B_-^T-\Lambda_-\,,\label{eq_imb_blk2}\\
        A_+^TA_-&=\;B_+B_-^T\,.
    \end{align}
    With these equalities, we know the following matrix is p.s.d., for any $\hat{\lambda}>\bar{\lambda}\geq 0$,
    \begin{align}
        \bmt
            B_+B_+^T+\hat{\lambda} I_{h-m} &B_+B_-^T\\
            B_-B_+^T& B_-B_-^T-\underline{\lambda}I_m
        \emt\overset{\eqref{eq_blk_lambda_pd}}{\succeq}&\; \bmt
            B_+B_+^T+\Lambda_+ &B_+B_-^T\\
            B_-B_+^T& B_-B_-^T-\Lambda_-
        \emt\nonumber\\
        =&\;\bmt A_+^T\\ A_-^T\emt\bmt A_+&A_-\emt\succeq 0\,.\label{eq_psd}
    \end{align}
    Since $B_+B_+^T+\hat{\lambda}I_{h-m}\succ 0$, positive semi-definiteness \eqref{eq_psd} is equivalent to
    \be
        B_-B_-^T-\underline{\lambda}I_m-B_-B_+^T(B_+B_+^T+\hat{\lambda}I_{h-m})^{-1}B_+B_-^T\succeq 0\,.\label{eq_psd_equiv}
    \ee
    Now we use Woodbury's Identity~\citep[0.7.4]{Horn:2012:MA:2422911}, which says for matrices $M,N,P$ with appropriate dimensions, we have
    \ben
        (M+P^TNP)^{-1}=M^{-1}-M^{-1}P^T(PM^{-1}P^T+N^{-1})^{-1}PM^{-1}\,,
    \een
    if all inverses exist. Let $M=I_{m},N=\hat{\lambda}^{-1}I_{h-m}, P=B_+$, we have
    \ben
        (I_m+\hat{\lambda}^{-1}B_+^TB_+)^{-1}=I_m-B_+^T(\hat{\lambda}I_{h-m}+B_+B_+^T)^{-1}B_+\,,
    \een
    which leads to
    \be
        B_-(I_m+\hat{\lambda}^{-1}B_+^TB_+)^{-1}B_-^T=B_-B_-^T-B_-B_+^T(\hat{\lambda}I_{h-m}+B_+B_+^T)^{-1}B_+B_-^T\,.\label{eq_woodbury_pd}
    \ee
    Using \eqref{eq_woodbury_pd}, we can rewrite \eqref{eq_psd_equiv} as
    \be 
        \underline{\lambda}I_m-B_-(I_m+\hat{\lambda}^{-1}B_+^TB_+)^{-1}B_-^T\preceq 0\label{eq_psd_equiv2}\,.
    \ee
    Consider the following matrix congruence
    \begin{align}
        &\; \bmt
            \underline{\lambda}I_{m} & B_-\\
            B_-^T& I_m+\hat{\lambda}^{-1}B_+^TB_+
        \emt\nonumber\\
        =&\;S_1\bmt \underline{\lambda}I_{m}-B_-(I_m+\hat{\lambda}^{-1}B_+^TB_+)^{-1}B_-^T & 0\\
        0 & I_m+\hat{\lambda}^{-1}B_+^TB_+\emt S_1^T\label{eq_mat_cong_blk1}\\
        =&\; S_2\bmt 
            \underline{\lambda}I_{m} & 0\\
            0 & I_m+\hat{\lambda}^{-1}B_+^TB_+-\underline{\lambda}^{-1}B_-^TB_-
        \emt S_2^T\label{eq_mat_cong_blk2}
    \end{align}
    where
    \ben
        S_1= \bmt 
            I_{m} & B_-(I_m+\hat{\lambda}^{-1}B_+^TB_+)^{-1}\\
            0 & I_m 
        \emt,\qquad S_2=\bmt 
            I_m & 0\\
            \underline{\lambda}^{-1}B_-^T & I_{m} 
        \emt\,,
    \een
    and $S_1,S_2$ are non-singular. By Sylvester's Intertia Theorem~\citep[Theorem 4.5.8]{Horn:2012:MA:2422911}, the block diagonal matrix shown in \eqref{eq_mat_cong_blk1} has exactly the same number of positive eigenvalues as the one shown in \eqref{eq_mat_cong_blk2}, and the number of positive eigenvalues is $m$, according to \eqref{eq_psd_equiv2}. Then for the block diagonal matrix in \eqref{eq_mat_cong_blk2}, we must have
    \ben
        I_m+\hat{\lambda}^{-1}B_+^TB_+-\underline{\lambda}^{-1}B_-^TB_-\preceq 0\,,
    \een
    hence
    \begin{align}
        0&\preceq\; -I_m-\hat{\lambda}^{-1}B_+^TB_++\underline{\lambda}^{-1}B_-^TB_-\nonumber\\
        0&\preceq\; -\hat{\lambda}\underline{\lambda}I_m-\underline{\lambda}B_+^TB_++\hat{\lambda}B_-^TB_-\nonumber\\
        \hat{\lambda}B_+^TB_+-\underline{\lambda}B_-^TB_-&\preceq\; -\hat{\lambda}\underline{\lambda}I_m-\underline{\lambda}B_+^TB_++\hat{\lambda}B_-^TB_-\nonumber\\
        &\;\qquad \qquad\qquad+\hat{\lambda}B_+^TB_+-\underline{\lambda}B_-^TB_-\nonumber\\
        \hat{\lambda}B_+^TB_+-\underline{\lambda}B_-^TB_-&\preceq\; -\hat{\lambda}\underline{\lambda}I_m+(\hat{\lambda}-\underline{\lambda})(B_+^TB_++B_-^TB_-)\nonumber\\
        \hat{\lambda}B_+^TB_+-\underline{\lambda}B_-^TB_-&\preceq\; -\hat{\lambda}\underline{\lambda}I_m+(\hat{\lambda}-\underline{\lambda})B^TB\,,\label{eq_b_cross}
    \end{align}
    where the last equivalence uses the fact $B^TB=B_+^TB_++B_-^TB_-$. This suggests that 
    \begin{align}
        B^TDB= B_+^T\Lambda_+B_+-B_-^T\Lambda_-B_-
        &\preceq \;\hat{\lambda}B_+^TB_+-\underline{\lambda}B_-^TB_-\nonumber\\
        &\overset{\eqref{eq_b_cross}}{\preceq}\; -\hat{\lambda}\underline{\lambda}I_m+(\hat{\lambda}-\underline{\lambda})B^TB\label{eq_BDB_ub}
    \end{align}
    Lastly, from \eqref{eq_lem_lb_w2_eq1} we have 
    \begin{align*}  
        \lambda_m^2(B^TB)=z_m^T(B^TB)^2z_m&\geq\; \sigma_m^2(AB)-z_m^TB^TDBz_m\\
        &\overset{\eqref{eq_BDB_ub}}{\geq}\; \sigma_m^2(AB)-\hat{\lambda}\underline{\lambda}+(\hat{\lambda}-\underline{\lambda})z^T_mB^TBz_m\\
        &=\; \sigma_m^2(AB)-\hat{\lambda}\underline{\lambda}+(\hat{\lambda}-\underline{\lambda})\lambda_m(B^TB)\,.
    \end{align*}
    This quadratic inequality w.r.t. $\lambda_m(B^TB)$ has nonnegative solutions
    \ben
        \lambda_m(B^TB)\geq \frac{\underline{\lambda}-\hat{\lambda}+\sqrt{(\underline{\lambda}-\hat{\lambda})^2+4\hat{\lambda}\underline{\lambda}+4\sigma_m^2(AB)}}{2}
        =\frac{-\hat{\lambda}+\underline{\lambda}+\sqrt{(\hat{\lambda}+\underline{\lambda})^2+4\sigma_m^2(AB)}}{2}\;\,.
    \een
    Since we can choose any $\hat{\lambda}>\bar{\lambda}\geq 0$, we have 
    \ben
        \lambda_m(B^TB)\geq \lim_{\hat{\lambda}\ra \bar{\lambda}}\frac{-\hat{\lambda}+\underline{\lambda}+\sqrt{(\hat{\lambda}+\underline{\lambda})^2+4\sigma_m^2(AB)}}{2}
        =\frac{-\bar{\lambda}+\underline{\lambda}+\sqrt{(\bar{\lambda}+\underline{\lambda})^2+4\sigma_m^2(AB)}}{2}\;\,.
    \een
    This is exactly \eqref{eq_lem_lb_w2}.
    
    (Note that when $\bar{\lambda}>0$, one can pick $\hat{\lambda}=\bar{\lambda}$ and obtain the desired bound directly. Taking the limit $\hat{\lambda}\ra \bar{\lambda}$ is necessary only when $\bar{\lambda}=0$).
\end{proof}

\section{Detailed analysis for the matrix factorization problem}\label{app_tightness}
Consider the gradient flow on $\tilde{\mathcal{L}}=\frac{1}{2}\|Y-UV^T\|^2_F$,\footnote{When $\Sigma_x=I_r$, $\mathcal{L}-\mathcal{L}^*=\frac{1}{2}\|W^TY-\Sigma_x^{1/2}U_1V^T\|_F^2=\frac{1}{2}\|\tilde{Y}-U_1V^T\|_F^2$ is exactly of this form.} where $Y\in\mathbb{R}^{r\times m},U\in\mathbb{R}^{r\times h},V\in\mathbb{R}^{h\times n}$. Still we define $E:=Y-UV^T$.
    
    We start with the exact expression for the instantaneous rate
    \begin{align*}
        -\frac{\dot{\tilde{\mathcal{L}}}}{\tilde{\mathcal{L}}}=-2\frac{\tr (E^T\dot{E})}{\|E\|^2_F}
        = 2\frac{\tr (E^T(U\dot{V}+\dot{U}V^T))}{\|E\|^2_F}
        =2\frac{\tr (E^T(UU^TE+EVV^T))}{\|E\|^2_F}\,.
    \end{align*}
    If we define the Hermitian linear operator $\mathcal{T}_{U,V}$ on $\mathbb{R}^{r\times m}$ as $\mathcal{T}_{U,V}E=UU^TE+EVV^T$. Then the instantaneous rate is actually a Rayleigh quotient
    $$
        -\frac{\dot{\tilde{\mathcal{L}}}}{\tilde{\mathcal{L}}}=2\frac{\lan E,\mathcal{T}_{U,V}E\ran_F}{\lan E,E\ran_F}\,,
    $$
    where $\lan \cdot,\cdot \ran_F$ is the Frobenius inner product on $\mathbb{R}^{r\times m}$. Notice that both $E$ and $\mathcal{T}_{U,V}$ depend on $U,V$ here. 
    
    Now our goal is find the best lower bound on the instantaneous rate, provided that the imbalance $D=U^TU-V^TV$ and product $W=UV^T$ is known to us (Recall that we can express the instantaneous rate exactly by the imbalance and product in the scalar case \eqref{eq_scalar_insta_r_dp}). That is, the following problem
    \begin{problem}\label{prob_best_bound}
        Suppose $h\geq \min\{r,m\}$. Given $Y\in\mathbb{R}^{r\times m}$, $D\in\mathbb{R}^{h\times h}$ and $W\in \mathbb{R}^{r\times m}$, find
        $$
            c^*(Y,D,W)=\min\left\{2\frac{\lan E,\mathcal{T}_{U,V}E\ran_F}{\lan E,E\ran_F}:U^TU-V^TV=D,UV^T=W\right\}\,.
        $$
    \end{problem}
    $c^*$ is the best bound we can obtain by knowing the imbalance $D$ and the product $W$, but it also depends on $Y$ because we have defined $E=Y-UV^T$. Problem \ref{prob_best_bound} is generally hard to solve except for very special cases, thus we consider a lower bound for $c^*$:
    \begin{align*}
        c^*(Y,D,W)&=\;\min\left\{2\frac{\lan E,\mathcal{T}_{U,V}E\ran_F}{\lan E,E\ran_F}:U^TU-V^TV=D,UV^T=W\right\}\\
        &\geq\;\min\left\{2\min_Y\frac{\lan E,\mathcal{T}_{U,V}E\ran_F}{\lan E,E\ran_F}:U^TU-V^TV=D,UV^T=W\right\}\\
        &=\; \min\left\{2\lambda_{\min}(\mathcal{T}_{U,V}):U^TU-V^TV=D,UV^T=W\right\}:=c(D,W)\,,
    \end{align*}
    Here the second equality is obtained by choosing $Y=E_{\min}+UV^T$ where $E_{\min}$ is the least eigenmatrix of $\mathcal{T}_{U,V}$. Moreover, one can show that~\citep{schacke2004kronecker} $$\lambda_{\min}(\mathcal{T}_{U,V})=\lambda_{r}(UU^T)+\lambda_{m}(VV^T)\,.$$
    This left us to consider the following problem
    \begin{problem}\label{prob_weaker_bound}
        Suppose $h\geq \min\{r,m\}$. Given $Y\in\mathbb{R}^{r\times m}$, $D\in\mathbb{R}^{h\times h}$ and $W\in \mathbb{R}^{r\times m}$, find
        $$
            c(D,W)=\min\left\{2(\lambda_{r}(UU^T)+\lambda_{m}(VV^T)):U^TU-V^TV=D,UV^T=W\right\}\,.
        $$
    \end{problem}
    Now following the results in Section \ref{ssec:conv_pf}, one obtain, by Lemma \ref{lem_lb_a_b},
    $$
        c(D,W)\geq -\Delta_++\sqrt{(\Delta_++\underline{\Delta})^2+4\sigma^2_{m}(W)}-\Delta_-+\sqrt{(\Delta_-+\underline{\Delta})^2+4\sigma^2_{r}(W)}\,.
    $$
    It turns out that this lower bound for $C(D,W)$ is tight in most cases. Formally speaking, we have
    \begin{proposition}\label{prop_app_tightness}
        Suppose $h\geq \min\{r,m\}$. When $r\neq m$, we have 
        $$
            c(D,W)= -\Delta_++\sqrt{(\Delta_++\underline{\Delta})^2+4\sigma^2_{m}(W)}-\Delta_-+\sqrt{(\Delta_-+\underline{\Delta})^2+4\sigma^2_{r}(W)}\,.
        $$
    \end{proposition}
    In summary, we have shown
    \begin{align*}
        -\frac{\dot{\tilde{\mathcal{L}}}}{\tilde{\mathcal{L}}}\geq c^*(Y,D,W)&\geq\; c(D,W)\\
        &\geq\;-\Delta_++\sqrt{(\Delta_++\underline{\Delta})^2+4\sigma^2_{m}(W)}-\Delta_-+\sqrt{(\Delta_-+\underline{\Delta})^2+4\sigma^2_{r}(W)}\,.
    \end{align*}
    Therefore, the bound in Proposition \ref{prop_lb_insta_rate}, in some sense, is not the best lower bound we can obtain. Potential improvement for the current results would be studying $c^*(Y,D,W)$ directly.
    
    We end this section with the proof for Proposition \ref{prop_app_tightness}.
    \begin{proof}[Proof of Proposition \ref{prop_app_tightness}]Assume $r>m$, the proof for the opposite case is identical.
    
    The statement in Proposition \ref{prop_app_tightness} is equivalent to that given $W\in\mathbb{R}^{r\times m}$ and a symmetric $D\in\mathbb{R}^{h\times h}$ with $\rank(D)\leq r+m$, there exist $U\in\mathbb{R}^{r\times h},V\in\mathbb{R}^{m\times h}$ such that 
    $$
        U^TU-V^TV=D,\ UV^T=W\,,
    $$
    and
    \begin{align*}
        &\;2(\lambda_{r}(UU^T)+\lambda_{m}(VV^T))=\\
        &\;\qquad -\Delta_++\sqrt{(\Delta_++\underline{\Delta})^2+4\sigma^2_{m}(W)}-\Delta_-+\sqrt{(\Delta_-+\underline{\Delta})^2+4\sigma^2_{r}(W)}\,.
    \end{align*}
    Now consider the SVD of $W$, $W=P\tilde{W}Q^T$ , where $P\in\mathcal{O}(r),Q\in\mathcal{O}(m)$ and $\tilde{W}\in\mathbb{R}^{r\times m}$ has singular values of $W$ on its main diagonal. Additionally, we have the eigendecomposition of $D$, $D=R\tilde{D}R^T$, where $R\in\mathcal{O}(h)$ and $\tilde{D}$ has eigenvalues of $D$ on its diagonal. Notice that there are at most $r+m$ non-zero eigenvalues of $D$.
    
    With the decomposition of $W$ and $D$. We only need to show that there exist $\tilde{U}\in\mathbb{R}^{r\times h},\tilde{V}\in\mathbb{R}^{m\times h}$ such that 
    $$
        \tilde{U}^T\tilde{U}-\tilde{V}^T\tilde{V}=\tilde{D},\ \tilde{U}\tilde{V}^T=\tilde{W}\,,
    $$
    and
    \begin{align*}
        &\;2(\lambda_{r}(\tilde{U}\tilde{U}^T)+\lambda_{m}(\tilde{V}\tilde{V}^T))=\\
        &\;\qquad -\Delta_++\sqrt{(\Delta_++\underline{\Delta})^2+4\sigma^2_{m}(W)}-\Delta_-+\sqrt{(\Delta_-+\underline{\Delta})^2+4\sigma^2_{r}(W)}\,.
    \end{align*}
    Then the statement on the existence of $U,V$ is true by setting $U=P\tilde{U}R^T$ and $V=Q\tilde{V}R^T$. Therefore, without loss of generality, we assume $W$ is main diagonal and $D$ diagonal. Furthermore, all main diagonal entries of $W$ are non-negative, denoted by $\sigma_i(W),i=1,\cdots,m$, and we reorder the diagonal entries of $D$ such that
    $$
    D= \bmt \Lambda_1&  &  & & & & \\
    &\ddots & && & & \\
     &  & \Lambda_m & & & &\\
     &  & &\lambda_{m+1}(D) & & &\\
     &  & &  &\ddots & &\\
     &  & & &  &\lambda_r(D) & \\
     &  & &  & & & 0\emt\,,$$
    where
    \begin{align*}
        &\;\Lambda_i=\bmt \lambda_i(D)& 0\\
        0& \lambda_{h-m+i}(D)\emt=\bmt \lambda_i(D)& 0\\
        0& -\lambda_{m-i+1}(-D)\emt\,.
    \end{align*}
    One can verify that this is indeed a valid reordering of the eigenvalues of $D$. Notice that under this reordering, we are pairing some eigenvalues of $D$, as seen in the definition of the $2\times 2$ blocks $\Lambda_i,i=1,\cdots,m$. Moreover, for any of these pairs $\Lambda_i$, we have one non-negative eigenvalue $\lambda_i(D)$ and one non-positive eigenvalue $\lambda_{h-m+i}(D)$.
    
    We claim that given such $W$ and $D$, one of the solution $U,V$ to $$
        U^TU-V^TV=D,\ UV^T=W\,,
    $$
    is of the form
    $$
        U=\bmt U_{\mathrm{nonzero}}& 0_{r\times (h-r-m)}\emt, U_{\mathrm{nonzero}}=\bmt 
        u_1^T & & & & &  \\
         & \ddots& & & &  \\
         & &u_m^T & & &  \\
         & & &\sqrt{\lambda_{m+1}(D)} & &  \\
         & & & &\ddots &  \\
         & & & & & \sqrt{\lambda_{r}(D)}
        \emt\,,
    $$
    $$
        V=\bmt V_{\mathrm{nonzero}}& 0_{m\times (h-2m)}\emt, V_{\mathrm{nonzero}}=\bmt 
        v_1^T & &\\
         & \ddots&  \\
         & &v_m^T 
        \emt\,,
    $$
    where $u_i,v_i,i=1,\cdots,m$ are $2\times 1$ vectors, and they necessarily satisfies for $i=1\cdots,m$,
    \be
        u_iu_i^T-v_iv_i^T=\Lambda_i,\ u_i^Tv_i=\sigma_i(W)\,.\label{eq_pf_tightness_block_eq}
    \ee
    Let $u_i=\bmt u_{i1}\\ u_{i2}\emt,v_i=\bmt v_{i1}\\ v_{i2}\emt$, \eqref{eq_pf_tightness_block_eq} yields $4$ quadratic equations in terms of $4$ unknowns $u_{i1},u_{i2},v_{i1},v_{i2}$\,. \eqref{eq_pf_tightness_block_eq} has at least one real solution whose exact expression is too complicated to write here, but for that real solution, we have
    $$
        u_i^Tu_i=\frac{\lambda_i(D)-\lambda_{m-i+1}(-D)+\sqrt{(\lambda_i(D)+\lambda_{m-i+1}(-D))^2+4\sigma^2_i(W)}}{2}\,,
    $$
    and 
    $$
        v_i^Tv_i=\frac{-\lambda_i(D)+\lambda_{m-i+1}(-D)+\sqrt{(\lambda_i(D)+\lambda_{m-i+1}(-D))^2+4\sigma^2_i(W)}}{2}\,.
    $$
    Notice that under this construction $\lambda_i(UU^T)=u_i^Tu_i,\lambda_i(VV^T)=v_i^Tv_i, i=1,\cdots,m$ and $\lambda_j(UU^T)=\lambda_j(D), j=m+1,\cdots,r$. Therefore,
    \begin{align*}
        &\;\lambda_r(UU^T)+\lambda_m(VV^T)=\\
        &\;\qquad \lambda_r(D)+\frac{-\lambda_1(D)+\lambda_{m}(-D)+\sqrt{(\lambda_1(D)+\lambda_{m}(-D))^2+4\sigma^2_m(W)}}{2}\,,
    \end{align*}
    which is exactly the lower bound we get using Lemma \ref{lem_lb_a_b}. Then the statement 
    \begin{align*}
        &\;2(\lambda_{r}(UU^T)+\lambda_{m}(VV^T))=\\
        &\;\qquad -\Delta_++\sqrt{(\Delta_++\underline{\Delta})^2+4\sigma^2_{m}(W)}-\Delta_-+\sqrt{(\Delta_-+\underline{\Delta})^2+4\sigma^2_{r}(W)}\,.
    \end{align*}
    follows from the definition of $\Delta_+,\Delta_-,\underline{\Delta}$.
    \end{proof}
    \section{Comparison with the NTK Initialization for wide single-hidden-layer linear networks}\label{app_comp_ntk}
    \setcounter{equation}{0}
    
    In Section~\ref{secc:wide_linear_net}, we analyzed implicit bias of wide single-hidden-layer linear networks under properly scaled random initialization. Our initialization for network weights $U,V$ is different from the typical setting in previous works~\citep{jacot2018neural,du2019width,arora2019exact}. In this section, we show that under our setting, the gradient flow is related to the NTK flow by 1) reparametrization and rescaling in time ; 2) proper scaling of the network output. The use of output scaling is also used in~\citet{arora2019exact}. 
    
    In this paper we work with a single-hidden-layer linear network defined as $f:\mathbb{R}^n\ra \mathbb{R}^m, f(x;V,U)=VU^Tx$, which is parametrized by $U,V$. Then we analyze the gradient flow on the loss function $\mathcal{L}(V,U)=\frac{1}{2}
    \lV Y-XUV^T\rV^2_F$, given the data and output matrix $X,Y$. Lastly, in Section 4.2, we initialize $U(0),V(0)$ such that all the entries are randomly drawn from $\mathcal{N}\lp 0,h^{-2\alpha}\rp$ ($1/4<\alpha\leq 1/2$), where $h$ is the hidden layer width.
    
    Now we define $\tilde{U}:=h^\alpha U,\tilde{V}:=h^\alpha V$, then the loss function can be written as
    \begin{align*}
        \mathcal{L}(V,U)=\tilde{\mathcal{L}}(\tilde{V},\tilde{U})=\frac{1}{2}
        \lV Y-\frac{1}{h^{2\alpha}}X\tilde{U}\tilde{V}^T\rV^2_F&=\;\frac{1}{2}
        \lV Y-\frac{\sqrt{m}}{h^{2\alpha-\frac{1}{2}}}\frac{1}{\sqrt{mh}}X\tilde{U}\tilde{V}^T\rV^2_F\\
        &=\;\frac{1}{2}\sum_{i=1}^n\lV y^{(i)}-\frac{\sqrt{m}}{h^{2\alpha-\frac{1}{2}}}\frac{1}{\sqrt{mh}}\tilde{V}\tilde{U}^Tx^{(i)}\rV^2_2\\
        &:=\; \sum_{i=1}^n\lV y^{(i)}-\frac{\sqrt{m}}{h^{2\alpha-\frac{1}{2}}}\tilde{f}(x;\tilde{V},\tilde{U})\rV^2_2
    \end{align*}
    Notice that $\tilde{f}(x;\tilde{V},\tilde{U})=\frac{1}{\sqrt{mh}}\tilde{V}\tilde{U}^Tx$ is the typical network discussed in previous works~\citep{jacot2018neural,du2019width,arora2019exact}. When all the entries of $U(0),V(0)$ are initialized randomly as $\mathcal{N}\lp 0,h^{-2\alpha}\rp$, the entries of $\tilde{U}(0),\tilde{V}(0)$ are random samples from $\mathcal{N}\lp 0,1\rp$, which is the typical choice of initialization for NTK analysis.
    
    However, the difference is that $\tilde{f}(x;\tilde{V},\tilde{U})$ is scaled by $\frac{\sqrt{m}}{h^{2\alpha-\frac{1}{2}}}$. In previous work showing non-asymptotic bound between wide neural networks and its infinite width limit~\citep[Theorem 3.2]{arora2019exact}, the wide neural network is scaled by a small constant $\kappa$ such that the prediction by the trained network is within $\epsilon$-distance to the one by the kernel predictor of its NTK. Moreover,~\citet{arora2019exact} suggests $\frac{1}{\kappa}$ should scale as $poly(\frac{1}{\epsilon})$, i.e., to make sure the trained network is arbitrarily close to the kernel predictor, $\kappa$ should be vanishingly small. In our setting, the random initialization implicitly enforces such a vanishing scaling $\frac{\sqrt{m}}{h^{2\alpha-\frac{1}{2}}}$, as the width of network increases.
    
    Lastly, we show that the gradient flow on $\mathcal{L}(V,U)$ only differs from the flow on $\tilde{\mathcal{L}}(\tilde{V},\tilde{U})$ by the time scale. 
    
    Suppose $U,V$ follows the gradient flow on $\mathcal{L}(V,U)$,  we have
    \begin{align}
        -\frac{1}{h^\alpha}\frac{\partial}{\partial U}\mathcal{L}(V,U)&=\;-\frac{1}{h^\alpha}X^T(Y-XUV^T)V\nonumber\\
        &=\;-\frac{1}{h^{2\alpha}}X^T\lp Y-\frac{1}{h^{2\alpha}}X\tilde{U}\tilde{V}^T\rp\tilde{V}=-\frac{\partial}{\partial \tilde{U}}\tilde{\mathcal{L}}(\tilde{V},\tilde{U})\,,\label{eq1_ntk_scale_grad}
    \end{align}
    and
    \begin{align}
        -\frac{1}{h^\alpha}\frac{\partial}{\partial V}\mathcal{L}(V,U)&=\;-\frac{1}{h^\alpha}(Y-XUV^T)^TXU\nonumber\\
        &=\;-\frac{1}{h^{2\alpha}}\lp Y-\frac{1}{h^{2\alpha}}X\tilde{U}\tilde{V}^T\rp^TX\tilde{U}=-\frac{\partial}{\partial \tilde{V}}\tilde{\mathcal{L}}(\tilde{V},\tilde{U})\,.\label{eq2_ntk_scale_grad}
    \end{align}
    From \eqref{eq1_ntk_scale_grad}, we have
    \begin{align}
        \dot{U}=-\frac{\partial}{\partial U}\mathcal{L}(V,U)
        \Leftrightarrow&\;\frac{1}{h^\alpha}\dot{\tilde{U}}=-\frac{\partial}{\partial U}\mathcal{L}(V,U)\nonumber\\
        \Leftrightarrow&\;\frac{1}{h^\alpha}\dot{\tilde{U}}=-h^\alpha\frac{\partial}{\partial \tilde{U}}\tilde{\mathcal{L}}(\tilde{V},\tilde{U})\nonumber\\
        \Leftrightarrow&\;\dot{\tilde{U}}=-h^{2\alpha}\frac{\partial}{\partial \tilde{U}}\tilde{\mathcal{L}}(\tilde{V},\tilde{U})\label{eq1_fast_than_ntk}\,,
    \end{align}
    Similarly from \eqref{eq2_ntk_scale_grad} we have
    \be
        \dot{V}=-\frac{\partial}{\partial V}\mathcal{L}(V,U)
        \Leftrightarrow\dot{\tilde{V}}=-h^{2\alpha}\frac{\partial}{\partial \tilde{V}}\tilde{\mathcal{L}}(\tilde{V},\tilde{U})\label{eq2_fast_than_ntk}\,.
    \ee
    From \eqref{eq1_fast_than_ntk} and \eqref{eq2_fast_than_ntk} we know that the gradient flow on $\mathcal{L}(V,U)$ w.r.t. time $t$ essentially runs the gradient flow on $\tilde{\mathcal{L}}(\tilde{V},\tilde{U})$ with an scaled-up rate by $h^{2\alpha}$.

\section{Proofs of Proposition \ref{prop_conv_stationary}, Lemma \ref{lem_si_no} and Theorem \ref{thm_asymp_conv_min_norm}}\label{app_pf_thm2}
\setcounter{equation}{0}

We start with the proof of Proposition \ref{prop_conv_stationary}.

\begin{proof}[Proof of Proposition \ref{prop_conv_stationary}]
    Since $c(0)>0$, for the gradient system \eqref{eq_gf_rp}, the states (parameters) $(U_1,V)$ converge either to an equilibrium point which minimizes the potential $\frac{1}{2}\|E\|_F^2=\mathcal{L}-\mathcal{L}^*$ or have its $l_2$-norm grow to infinity~\citep{hirsch1974differential}.
    
    Consider the following dynamics
    \be
        \frac{d}{dt}\bmt V\\ U_1\emt = \underbrace{\bmt 0&  E^T\Sigma^{1/2}_x\\
        \Sigma^{1/2}_xE & 0\emt}_{:=A_Z}\underbrace{\bmt V\\ U_1\emt}_{:=Z}\,,\label{eq_full_dym_concate}
    \ee
    which can be viewed as a time-variant linear system. Notice that by \citet[Theorem 7.3.3]{Horn:2012:MA:2422911}, we have $\|A_Z\|_2=\|\Sigma_x^{1/2}E\|_2$.
    
    From \eqref{eq_full_dym_concate}, we have
    \begin{align*}
        \frac{d}{dt}\|Z\|_F^2&=\;2\tr\lp Z^TA_ZZ\rp\\
        &=\;2\tr\lp ZZ^TA_Z\rp\\
        &\leq 2\|A_Z\|_2\tr\lp ZZ^T\rp\\
        &=\;2\|\Sigma^{1/2}_xE\|_2\|Z\|_F^2\\
        &\leq \; 2\lambda_1^{1/2}(\Sigma_x)\|E\|_2\|Z\|_F^2
        \\
        &\leq\; 2\lambda_1^{1/2}(\Sigma_x)\|E\|_F\|Z\|_F^2\,.
    \end{align*}
    By Gr\"onwall's inequality~\citep{gronwall1919}, we have
    \ben
        \|Z(t)\|_F^2\leq \exp\lp\int_0^t2\lambda_1^{1/2}(\Sigma_x)\|E(\tau)\|_F d\tau\rp\|Z(0)\|_F^2\,.
    \een
    Finally, by Theorem \ref{thm_conv_lin_net}, we have $\|E(t)\|_F\leq \exp\lp -\lambda_n(\Sigma_x)c(0)t/2\rp\|E(0)\|_F,\ \forall t>0$, since $\|E\|_F=\sqrt{2(\mathcal{L}-\mathcal{L}^*)}$, which leads to
    \begin{align*}
         &\;\exp\lp\int_0^t2\lambda_1^{1/2}(\Sigma_x)\|E(\tau)\|_F d\tau\rp\\
         \leq&\; \exp\lp 2\lambda_1^{1/2}(\Sigma_x)\|E(0)\|_F \lp\int_0^t\exp\lp -\lambda_n(\Sigma_x)c(0)\tau/2\rp d\tau\rp\rp\nonumber\\
         \leq&\;
         \exp\lp 2\lambda_1^{1/2}(\Sigma_x)\|E(0)\|_F \lp\int_0^\infty\exp\lp -\lambda_n(\Sigma_x)c(0)\tau/2\rp d\tau\rp\rp\nonumber\\
         =&\;\exp\lp \frac{4\lambda_1^{1/2}(\Sigma_x)}{c(0)\lambda_n(\Sigma_x)}\|E(0)\|_F \rp\,.\label{eq_bd_z_t}
    \end{align*}
    Therefore we have
    \ben
        \|Z(t)\|_F^2\leq \exp\lp \frac{4\lambda_1^{1/2}(\Sigma_x)}{c(0)\lambda_n(\Sigma_x)}\|E(0)\|_F \rp\|Z(0)\|_F^2\,,
    \een
    which implies that the trajectory $V(t),U_1(t), t>0$ is bounded, i.e. its $l_2$-norm can not grow to infinity, then it has to converge to some equilibrium point $(V(\infty),U_1(\infty))$ such that its potential is zero, i.e., $E(V(\infty),U_1(\infty))=0$.
\end{proof}
Now we turn to prove Lemma \ref{lem_si_no} and Theorem \ref{thm_asymp_conv_min_norm}. We need a basic result in random matrix theory
\begin{lemma_sec}\label{thm_sv_rnd_mat}
    Given $m,n\in\mathbb{N}$ with $m\leq n$. Let $A$ be an $n\times m$ random matrix with i.i.d. standard normal entries $A_{ij}\sim \mathcal{N}\lp 0,1\rp$. For $\delta>0$, with probability at least $1-2\exp(-\delta^2)$, we have
    \ben
        \sqrt{n}-(\sqrt{m}+\delta)\leq \sigma_m(A)\leq \sigma_1(A)\leq \sqrt{n}+(\sqrt{m}+\delta)\,.
    \een
\end{lemma_sec}
The proof can be found in~\citet[Theorem 2.13]{davidson2001local}. We also need the following inequality.
\begin{lemma_sec}\label{lem_weyl_ineq_deriv}
    Let $A\in \mathbb{R}^{k\times n},B\in \mathbb{R}^{n\times m}$. Suppose $n\leq m$, then
    $$\sigma_i(A)\sigma_n(B)\leq \sigma_i(AB)\,,$$
    for $1\leq i\leq \min\{k,n\}$.
\end{lemma_sec}
\begin{proof}
    We start with the case where $k=n$. When $\sigma_n(B^T)=0$, the result is trivial. When $\sigma_n(B^T)\neq 0$, we have $BB^\dagger=I$, where $B^\dagger$ is the Moore–Penrose inverse of $B$. By Weyl's inequality~\citep[7.3.P16]{Horn:2012:MA:2422911}, it follows that
    \ben
        \sigma_i(A)\leq \sigma_i(AB)\sigma_1(B^\dagger),\ \forall 1\leq i\leq n\,.
    \een
    Since $\sigma_1(B^\dagger)=\sigma^{-1}_n(B)$, we get the desired inequality.
    
    When $k>n$, we have $\forall 1\leq i\leq n$,
    \ben
        \sigma_i(A)=\sigma_i\lp\bmt A & 0_{k\times (k-n)}\emt\rp
        \leq \sigma_i\lp AB\rp\sigma_1(\bmt B^\dagger & 0_{m\times (k-n)}\emt)
        =\sigma_i(AB)\sigma_1(B^\dagger)\,,
    \een
    which still leads to the desired result.
    
    When $k<n$, consider replacing $A$ with $\bmt A\\ 0_{(n-k)\times n}\emt$, we have $\forall 1\leq i\leq k$,
    \ben
        \sigma_i(A)\sigma_n(B)=\sigma_i\lp\bmt A\\ 0_{(n-k)\times n}\emt\rp\sigma_n(B)
        \leq \sigma_i\lp\bmt AB\\ 0_{(n-k)\times m}\emt\rp=\sigma_i(AB)\,.
    \een
\end{proof}


Now we are ready to prove Lemma \ref{lem_si_no}.
\begin{customlem}{10}[restated]\label{lem_si_no_alpha} Let $\frac{1}{4}<\alpha\leq \frac{1}{2}$. Given data matrix $X$. $\forall \delta\in(0,1)$, $\forall h>h_0=poly\lp m,n,\frac{1}{\delta}\rp$, with probability at least $1-\delta$ over random initialization with $[U(0)]_{ij},[V(0)]_{ij}\sim\mathcal{N}(0,h^{-2\alpha})$, the following conditions hold:
    \begin{enumerate}[leftmargin=4mm]
        \item (Sufficient level of imbalance)
        \begin{align}
            \underline{\lambda}_+(0)+\underline{\lambda}_-(0)>h^{1-2\alpha}\,,\label{eq_si_no1}
        \end{align}
        where $\underline{\lambda}_+,\underline{\lambda}_-$ are defined in \eqref{eq_def_eigs2}.
        \item (Approximate orthogonality)
        \be
        \lV\bmt V(0)U_2^T(0)\\
        U_1(0)U_2^T(0)\emt\rV_F\leq  2\sqrt{m+r}\frac{\sqrt{m+n}+\frac{1}{2}\log \frac{2}{\delta}}{h^{2\alpha-\frac{1}{2}}}\,,\label{eq_si_noeq2}
        \ee
        \be
        \lV U_1(0)V^T(0)\rV_F\leq 2 \sqrt{m}\frac{\sqrt{m+n}+\frac{1}{2}\log \frac{2}{\delta}}{h^{2\alpha-\frac{1}{2}}}\,.\label{eq_si_no_}
        \ee
    \end{enumerate}
\end{customlem}
\begin{proof}[Proof of Lemma \ref{lem_si_no_alpha}]
    For readability we simply write $U(0),U_1(0),U_2(0),V(0),D(0)$ as $U,U_1,U_2,V,D$.
    
    Consider the matrix $\bmt V^T &U^T\emt$ which is $h\times (m+n)$. Apply Lemma \ref{thm_sv_rnd_mat} to matrix $A=h^\alpha\bmt V^T &U^T\emt$, with probability at least $1-\delta$, we have
    \ben 
        \sigma_{m+n}(h^\alpha\bmt V^T &U^T\emt)\geq \sqrt{h}-\lp\sqrt{m+n}+\delta\rp\,,
    \een
    which leads to
    \be 
         \sigma_{m+n}(\bmt V^T &U^T\emt)\geq h^{\frac{1}{2}-\alpha}-\frac{\sqrt{m+n}+\frac{1}{2}\log \frac{2}{\delta}}{h^\alpha}\,.\label{eq_sv_bd_concat_mat}
    \ee
    Regarding the first inequality, we write the imbalance as
    \ben
        U_1^TU_1-V^TV=\bmt V^T & U_1^T\emt\bmt -V  \\U_1\emt= \bmt V^T & U^T\emt\bmt -I_m &0\\
        0& \Phi_1\Phi_1^T\emt\bmt V \\ U\emt\,.
    \een
    For $h>\lp\sqrt{m+n}+\frac{1}{2}\log \frac{2}{\delta}\rp^2$, assume event \eqref{eq_sv_bd_concat_mat} happens, then $$\sigma_{m+n}\lp\bmt V^T & U^T\emt\rp\geq h^{\frac{1}{2}-\alpha}-\frac{\sqrt{m+n}+\frac{1}{2}\log \frac{2}{\delta}}{h^\alpha}>0\,,$$ hence we have
    \begin{align*}
        \sigma_{r+m}(D)=&\;\sigma_{r+m}(U_1^TU_1-V^TV)\\
        =&\;\sigma_{r+m}\lp \bmt V^T & U^T\emt\bmt -I_m &0\\
        0& \Phi_1\Phi_1^T\emt\bmt V \\ U\emt\rp\\
        (\text{Lemma }\ref{lem_weyl_ineq_deriv}) \geq &\;\sigma_{r+m}\lp \bmt V^T & U^T\emt\bmt -I_m &0\\
        0& \Phi_1\Phi_1^T\emt\rp\sigma_{m+n}\lp\bmt V \\ U\emt\rp\\
        = &\;\sigma_{r+m}\lp \bmt -I_m &0\\
        0& \Phi_1\Phi_1^T\emt\bmt V \\ U\emt\rp\sigma_{m+n}\lp\bmt V \\ U\emt\rp\\
        (\text{Lemma }\ref{lem_weyl_ineq_deriv}) \geq &\; \sigma_{r+m}\lp \bmt -I_m &0\\
        0& \Phi_1\Phi_1^T\emt\rp \sigma^2_{m+n}\lp\bmt V \\ U\emt\rp\\
        =&\;\sigma_{r+m}\lp \bmt -I_m &0\\
        0& \Phi_1\Phi_1^T\emt\rp \sigma^2_{m+n}\lp\bmt V^T & U^T\emt\rp\\
        =&\;\sigma^2_{m+n}\lp\bmt V^T & U^T\emt\rp\,,
    \end{align*}
    where the last equality is due to the fact that $\bmt -I_m &0\\
    0& \Phi_1\Phi_1^T\emt$ has exactly $r+m$ non-zero singular value and all of them are $1$.
    
    We further assume $h>16\lp\sqrt{m+n}+\frac{1}{2}\log \frac{2}{\delta}\rp^2$, conditioned on event \eqref{eq_sv_bd_concat_mat}, with probability $1$ we have
    \begin{align}
        \sigma_{r+m}(D) \geq &\; \sigma^2_{m+n}\lp\bmt V^T & U^T\emt\rp\nonumber\\
        \geq &\; \lp h^{\frac{1}{2}-\alpha}-\frac{\sqrt{m+n}+\frac{1}{2}\log \frac{2}{\delta}}{h^\alpha}\rp^2\nonumber\\
        = &\;h^{1-2\alpha}-2 \frac{\sqrt{m+n}+\frac{1}{2}\log \frac{2}{\delta}}{h^{2\alpha-\frac{1}{2}}}+\lp \frac{\sqrt{m+n}+\frac{1}{2}\log \frac{2}{\delta}}{h^\alpha}\rp^2\nonumber\\
        >&\;h^{1-2\alpha}-2 \frac{\sqrt{m+n}+\frac{1}{2}\log \frac{2}{\delta}}{h^{2\alpha-\frac{1}{2}}}\geq \frac{1}{2}h^{1-2\alpha}\,.\label{eq_event_imbalance}
    \end{align}
    
    Lastly, due to the minimax property of symmetric matrix~\citep[Theorem 4.2.6]{Horn:2012:MA:2422911}, we have
    \begin{align*}
        \lambda_{r+1}(D) &=\;\min_{\mathrm{dim}(S)=h-r}\max_{0\neq x\in S}\frac{x^TD x}{x^Tx}\\
        (\mathrm{dim}(\mathrm{ker}(U_1))\geq h-r)&\leq\; \min_{\substack{S\subseteq \mathrm{ker}(U_1)\\\mathrm{dim}(S)=h-r} }\max_{0\neq x\in S}\frac{x^TD x}{x^Tx}\\
        &=\; \min_{\substack{S\subset \mathrm{ker}(U_1)\\\mathrm{dim}(S)=r} }\max_{0\neq x\in S}\frac{x^T(-V^TV)x}{x^Tx}\leq 0\,,
    \end{align*}
    and
    \begin{align*}
        \lambda_r(D) &=\;\max_{\mathrm{dim}(S)=r}\min_{0\neq x\in S}\frac{x^TDx}{x^Tx}\\
        \lp\mathrm{dim}(\mathrm{ker}(V(0)))\geq h-m\geq r\rp&\geq\; \max_{\substack{S\subseteq \mathrm{ker}(V(0))\\\mathrm{dim}(S)=r} }\min_{0\neq x\in S}\frac{x^TDx}{x^Tx}\\
        &=\; \max_{\substack{S\subset \mathrm{ker}(V(0))\\\mathrm{dim}(S)=r} }\min_{0\neq x\in S}\frac{x^TU_1^TU_1x}{x^Tx}\geq 0\,.
    \end{align*}
    Similarly, we have
    \begin{align*}
        \lambda_{m+1}(-D)
        &\leq\; \min_{\substack{S\subseteq \mathrm{ker}(V)\\\mathrm{dim}(S)=h-m} }\max_{0\neq x\in S}\frac{x^T(-U_1^TU_1)x}{x^Tx}\leq 0\,,
    \end{align*}
    and 
    \begin{align*}
        \lambda_m(-D)
        &\geq\; \max_{\substack{S\subseteq \mathrm{ker}(U_1(0))\\\mathrm{dim}(S)=m} }\min_{0\neq x\in S}\frac{x^TV^TVx}{x^Tx}\geq 0\,.
    \end{align*}
    These inequalities together imply
    \ben
        \min\{\lambda_{r}(D),\lambda_{m}(-D)\}=\sigma_{r+m}(D)\,.
    \een
    Here we also use the fact that $D$ is symmetric. Now by \eqref{eq_event_imbalance}, we immediately obtain that conditioned on event \eqref{eq_sv_bd_concat_mat}, with probability $1$, the following holds,

    \begin{align*}
        \underline{\lambda}_++\underline{\lambda}_-=\lambda_{r}(D)+\lambda_{m}(-D)&\geq\;2\sigma_{r+m}(D)\geq h^{1-2\alpha}\,,
    \end{align*}
    which is exactly \eqref{eq_si_no1}.
    
    Regarding the second and third inequality, using the fact that $$\|A\|_F\leq \sqrt{\min\{n,m\}}\|A\|_2, \ \forall A\in\mathbb{R}^{n\times m}\,,$$
    we have
    \begin{align*}
        \frac{1}{\sqrt{m}}\lV U_1V^T\rV_F\leq \lV U_1V^T\rV_2
        =&\; \lV \bmt 0& \Phi_1^T\emt \bmt V\\
        U\emt\bmt V^T &U^T\emt \bmt I_m\\ 0\emt\rV_2\\
        =&\; \lV \bmt 0& \Phi_1^T\emt \lp\bmt V\\
        U\emt\bmt V^T &U^T\emt-\eta I_{m+n}\rp \bmt  I_m\\ 0\emt\rV_2\\
        \leq &\; \lV \bmt V\\
        U\emt\bmt V^T &U^T\emt-\eta I_{m+n}\rV_2\,, \text{for any } \eta\in\mathbb{R}\,,
    \end{align*}
    where the second equality is due to the fact that $\bmt 0& \Phi_1^T\emt \bmt I_m\\ 0\emt=0$. And
    
    \begin{align*}
        \frac{1}{\sqrt{m+r}}\lV\bmt VU_2^T\\
        U_1U_2^T\emt\rV_F\leq  \lV\bmt VU_2^T\\
        U_1U_2^T\emt\rV_2
        =&\; \lV \bmt I_m &0\\
        0&\Phi_1^T\emt\bmt V\\
        U\emt\bmt V^T &U^T\emt\bmt 0\\ \Phi_2\emt\rV_2\\
        =&\; \lV \bmt I_m &0\\
        0&\Phi_1^T\emt\lp\bmt V\\
        U\emt\bmt V^T &U^T\emt-\eta I_{m+n}\rp\bmt 0\\ \Phi_2\emt\rV_2\\
        \leq &\; \lV \bmt V\\
        U\emt\bmt V^T &U^T\emt-\eta I_{m+n}\rV_2\,, \text{for any } \eta\in\mathbb{R}\,,
    \end{align*}
    where the second equality is due to the fact that $\bmt I_m &0\\
        0&\Phi_1^T\emt\bmt 0\\  \Phi_2\emt=0$.
    Notice that
    \begin{align*}
        \lV \bmt V\\
        U\emt\bmt V^T &U^T\emt-\eta I_{m+n}\rV_2 = \max_i\lv \sigma_i^2(\bmt V^T&U^T\emt) -\eta\rv\,.
    \end{align*}

    Again we let $h>\lp\sqrt{m+n}+\frac{1}{2}\log \frac{2}{\delta}\rp^2$. When event \eqref{eq_sv_bd_concat_mat} happens, all $\sigma^2_i(\bmt V^T & U^T\emt)$ are within the interval $\left[ \lp h^{\frac{1}{2}-\alpha}-\frac{\sqrt{m+n}+\frac{1}{2}\log \frac{2}{\delta}}{h^\alpha}\rp^2,\lp h^{\frac{1}{2}-\alpha}-\frac{\sqrt{m+n}+\frac{1}{2}\log \frac{2}{\delta}}{h^\alpha}\rp^2\right]$. Since the choice of $\eta$ is arbitrary, we pick \be
    \eta=h^{1-2\alpha}+\lp\frac{\sqrt{m+n}+\frac{1}{2}\log \frac{2}{\delta}}{h^\alpha}\rp^2\label{eq_lem_si_no_1}\,,\ee which is the mid-point of this interval, then we have
    \begin{align*}
        &\;\max_i\lv \sigma_i^2(\bmt V^T&U^T\emt) -\eta\rv\nonumber\\
       \leq &\; \max\lb \lv\lp h^{\frac{1}{2}-\alpha}-\frac{\sqrt{m+n}+\frac{1}{2}\log \frac{2}{\delta}}{h^\alpha}\rp^2-\eta\rv,\lv\lp h^{\frac{1}{2}-\alpha}+\frac{\sqrt{m+n}+\frac{1}{2}\log \frac{2}{\delta}}{h^\alpha}\rp^2-\eta\rv \rb\nonumber\\
       &\;\text{($\eta$ is the mid-point)}\\
       \leq &\; \lv\lp h^{\frac{1}{2}-\alpha}-\frac{\sqrt{m+n}+\frac{1}{2}\log \frac{2}{\delta}}{h^\alpha}\rp^2-h^{1-2\alpha}-\lp\frac{\sqrt{m+n}+\frac{1}{2}\log \frac{2}{\delta}}{h^\alpha}\rp^2\rv\nonumber\\
        =&\;2\frac{\sqrt{m+n}+\frac{1}{2}\log \frac{2}{\delta}}{h^{2\alpha-\frac{1}{2}}}
    \end{align*}
    Therefore, when $h>\lp\sqrt{m+n}+\frac{1}{2}\log \frac{2}{\delta}\rp^2$, conditioned on event \eqref{eq_sv_bd_concat_mat}, with probability $1$, we have
    \begin{align}
        &\;\lV U_1V^T\rV_F\leq \sqrt{m}\lV \bmt V\\
        U\emt\bmt V^T &U^T\emt-\eta I_{m+n}\rV_2\leq 2\sqrt{m}\frac{\sqrt{m+n}+\frac{1}{2}\log \frac{2}{\delta}}{h^{2\alpha-\frac{1}{2}}}\,,\nonumber\\
        \mathrm{and}
        &\;\lV\bmt VU_2^T\\
        U_1U_2^T\emt\rV_F\leq \sqrt{m+r}\lV \bmt V\\
        U\emt\bmt V^T &U^T\emt-\eta I_{m+n}\rV_2\leq 2\sqrt{m+r}\frac{\sqrt{m+n}+\frac{1}{2}\log \frac{2}{\delta}}{h^{2\alpha-\frac{1}{2}}}\,,\label{eq_event_zero_init}
    \end{align}
    where we choose $\eta$ as in~\eqref{eq_lem_si_no_1}. 
    
    When $h>h_0=16\lp\sqrt{m+n}+\frac{1}{2}\log \frac{2}{\delta}\rp^2$ and conditioned on event \eqref{eq_sv_bd_concat_mat}, events \eqref{eq_event_imbalance} and \eqref{eq_event_zero_init} happen with probability $1$, hence the probability that both \eqref{eq_event_imbalance} and \eqref{eq_event_zero_init} happen is at least the probability of event \eqref{eq_sv_bd_concat_mat}, which is at least $1-\delta$.
\end{proof}

With Lemma \ref{lem_si_no_alpha}, we can prove Theorem \ref{thm_asymp_conv_min_norm}.
\begin{customthm}{2}[restated]\label{thm_asymp_conv_min_norm_alpha}
    Let $\frac{1}{4}<\alpha\leq \frac{1}{2}$. Let $V(t),U(t), t>0$ be the trajectory of the continuous dynamics \eqref{eq_gf_rp} starting from some $V(0), U(0)$. Then, $\exists C>0$, such that $\forall \delta\in(0,1), \forall h>h_0^{1/(4\alpha-1)}$ with $h_0=poly\lp m,n,\frac{1}{\delta},\frac{\lambda_1(\Sigma_x)}{\lambda_r^3(\Sigma_x)}\rp$, with probability $1-\delta$ over random initializations  with $[U(0)]_{ij},[V(0)]_{ij}\sim\mathcal{N}(0,h^{-2\alpha})$, we have
        \be
            \|U(\infty)V^T(\infty)-\hat{\Theta}\|_2\leq 2C^{1/h^{1-2\alpha}}\sqrt{m+r}\frac{\sqrt{m+n}+\frac{1}{2}\log \frac{2}{\delta}}{h^{2\alpha-\frac{1}{2}}}\,.
        \ee
        Here $C=\exp\lp 1+ \frac{\lambda_1^{1/2}(\Sigma_x)}{\lambda_r(\Sigma_x)}\|Y\|_F\rp$, which depends on the data $X,Y$.
\end{customthm}
\begin{proof}[Proof of Theorem \ref{thm_asymp_conv_min_norm_alpha}]
    From Corollary \ref{col_conv_imb} and Proposition \ref{prop_conv_stationary}, the stationary point $U(\infty),V(\infty)$ satisfy
    \ben
        U_1(\infty)V^T(\infty)=\Phi_1^T\hat{\Theta},\quad U_2(\infty)=U_2(0)\,, 
    \een
    provided that level of imbalance $\underline{\lambda}_++\underline{\lambda}_-$ is non-zero, which is guaranteed with high probability by Lemma \ref{lem_si_no_alpha}.
    Hence we have
    \begin{align*}
        \|U(\infty)V^T(\infty)-\hat{\Theta}\|_2 &=\; \|\Phi_1U_1(\infty)V^T(\infty)+\Phi_2U_2(\infty)V^T(\infty)-\hat{\Theta}\|_2\\
        &=\; \|\Phi_1\Phi_1^T\hat{\Theta}+\Phi_2U_2(\infty)V^T(\infty)-\hat{\Theta}\|_2\\
        &=\;\|\Phi_2U_2(\infty)V^T(\infty)\|_F\\
        &=\;\|\Phi_2U_2(0)V^T(\infty)\|_F= \|U_2(0)V^T(\infty)\|_2\leq \|U_2(0)V^T(\infty)\|_F\,.
    \end{align*}
    Consider the following dynamics
    \be
        \frac{d}{dt}\bmt VU_2^T\\ U_1U_2^T\emt = \underbrace{\bmt 0&  E^T\Sigma^{1/2}_x\\
        \Sigma^{1/2}_xE & 0\emt}_{:=A_Z}\underbrace{\bmt VU_2^T\\ U_1U_2^T\emt}_{:=Z}\,,\label{eq_proj_dym_app2}
    \ee
    which can be viewed as a time-variant linear system, and in particular, by~\citet[Theorem 7.3.3]{Horn:2012:MA:2422911}, we have $\|A_Z\|_2=\|\Sigma_x^{1/2}E\|_2$. Notice that here the $Z$ is different from the one in the proof for Proposition \ref{prop_conv_stationary}.
    
    From \eqref{eq_proj_dym_app2}, we have
    \begin{align*}
        \frac{d}{dt}\|Z\|_F^2&=\;2\tr\lp Z^TA_ZZ\rp\\
        &=\;2\tr\lp ZZ^TA_Z\rp\\
        &\leq 2\|A_Z\|_2\tr\lp ZZ^T\rp\\
        &=\;2\|\Sigma^{1/2}_xE\|_2\|Z\|_F^2\\
        &\leq \; 2\lambda_1^{1/2}(\Sigma_x)\|E\|_2\|Z\|_F^2\leq 2\lambda_1^{1/2}(\Sigma_x)\|E\|_F\|Z\|_F^2\,.
    \end{align*}
    By Gr\"onwall's inequality~\citep{gronwall1919}, we have $\forall t\geq 0$,
    \begin{align}
        &\;\|Z(t)\|_F^2\leq \exp\lp\int_0^t2\lambda_1^{1/2}(\Sigma_x)\|E(\tau)\|_F d\tau\rp\|Z(0)\|_F^2\nonumber\\
        \Ra &\; \|Z(t)\|_F\leq \exp\lp\int_0^t\lambda_1^{1/2}(\Sigma_x)\|E(\tau)\|_F d\tau\rp\|Z(0)\|_F\label{eq_thm_asymp_0}
    \end{align}
   
    Using Lemma \ref{lem_si_no_alpha}, for $h>h_0':= 16\lp\sqrt{m+n}+\frac{1}{2}\log \frac{2}{\delta}\rp^2$, with probability at least $1-\delta$ we have all the following.
    \begin{align}
        \underline{\lambda}_+(0)+\underline{\lambda}_-(0)&>\;h^{1-2\alpha}\,.\label{eq_thm_asymp_1}\\
        \lV U_1(0)V^T(0)\rV&\leq\;2\sqrt{m}\frac{\sqrt{m+n}+\frac{1}{2}\log \frac{2}{\delta}}{h^{2\alpha-\frac{1}{2}}}\,,\label{eq_thm_asymp_3}\\
        \|Z(0)\|_F =\lV\bmt V(0)U_2^T(0)\\
        U_1(0)U_2^T(0)\emt\rV_F&\leq\; 2\sqrt{m+r}\frac{\sqrt{m+n}+\frac{1}{2}\log \frac{2}{\delta}}{h^{2\alpha-\frac{1}{2}}}\label{eq_thm_asymp_2}
    \end{align}
    
    By Corollary \ref{col_conv_imb}, we have 
    \ben
        \|E(t)\|_F^2\leq \exp\lp -\lambda_r(\Sigma_x)c'(0)t\rp\|E(0)\|_F^2\,,
    \een
    where $c'(0) = 2(\underline{\lambda}_+(0)+\underline{\lambda}_-(0))$, then by \eqref{eq_thm_asymp_1}, we have
    \begin{align*}
        &\;\|E(t)\|_F^2\leq \exp\lp -2h^{1-2\alpha}\lambda_r(\Sigma_x)t\rp\|E(0)\|_F^2\\
        \Ra &\;\|E(t)\|_F\leq \exp\lp -h^{1-2\alpha}\lambda_r(\Sigma_x)t\rp\|E(0)\|_F\,.
    \end{align*}
    Finally, from \eqref{eq_thm_asymp_0}, we have
    \begin{align}
         \|Z(t)\|_F &\leq\; \exp\lp\int_0^t\lambda_1^{1/2}(\Sigma_x)\|E(\tau)\|_F d\tau\rp\|Z(0)\|_F\nonumber\\
         &\leq\; \exp\lp\lambda_1^{1/2}(\Sigma_x)\|E(0)\|_F \lp\int_0^t\exp\lp -h^{1-2\alpha}\lambda_r(\Sigma_x)\tau\rp d\tau\rp\rp\|Z(0)\|_F\nonumber\\
         &\leq\;
         \exp\lp\lambda_1^{1/2}(\Sigma_x)\|E(0)\|_F \lp\int_0^\infty\exp\lp -h^{1-2\alpha}\lambda_r(\Sigma_x)\tau\rp d\tau\rp\rp\|Z(0)\|_F\nonumber\\
         &=\;\exp\lp \frac{\lambda_1^{1/2}(\Sigma_x)}{h^{1-2\alpha}\lambda_r(\Sigma_x)}\|E(0)\|_F \rp\|Z(0)\|_F\,.\label{eq_bd_z_t}
    \end{align}
    The initial error depends on the initialization but can be upper bounded as
    \begin{align*}
        \|E(0)\|_F &=\; \|W^TY-\Sigma_x^{-1/2}U_1(0)V^T(0)\|_F\\
        &\leq\; \|W^TY\|_F+\|\Sigma_x^{-1/2}U_1(0)V^T(0)\|_F\\
        &\leq\; \|Y\|_F+\lambda_r^{-1/2}(\Sigma_x)\|U_1(0)V^T(0)\|_F
    \end{align*}
    then we can write \eqref{eq_bd_z_t} as
    \begin{align}
        \|Z(t)\|_F &\leq\;\exp\lp \frac{\lambda_1^{1/2}(\Sigma_x)}{h^{1-2\alpha}\lambda_r(\Sigma_x)}\|Y\|_F\rp \exp\lp \frac{\lambda_1^{1/2}(\Sigma_x)}{h^{1-2\alpha}\lambda_r^{3/2}(\Sigma_x)}\|U_1(0)V^T(0)\|_F\rp \|Z(0)\|_F\nonumber\\
        &=\;\lhp\exp\lp \frac{\lambda_1^{1/2}(\Sigma_x)}{\lambda_r(\Sigma_x)}\|Y\|_F\rp \exp\lp \frac{\lambda_1^{1/2}(\Sigma_x)}{\lambda_r^{3/2}(\Sigma_x)}\|U_1(0)V^T(0)\|_F\rp\rhp^{1/h^{1-2\alpha}} \|Z(0)\|_F
        \,.\label{eq_bd_z_t_interm}
    \end{align}
    For the second exponential, we let  $h_0:=\max\lb h_0',4\frac{\lambda_1(\Sigma_x)}{\lambda_r^{3}(\Sigma_x)}m\lp \sqrt{m+n}+\frac{1}{2}\log \frac{2}{\delta}\rp^2\rb$, then $\forall h>h_0^{1/(4\alpha-1)}$, by \eqref{eq_thm_asymp_3} we have
    \be
        \exp\lp \frac{\lambda_1^{1/2}(\Sigma_x)}{\lambda_r^{3/2}(\Sigma_x)}\|U_1(0)V^T(0)\|_F\rp \leq \exp\lp 2\frac{\lambda_1^{1/2}(\Sigma_x)}{\lambda_r^{3/2}(\Sigma_x)}\sqrt{m}\frac{\sqrt{m+n}+\frac{1}{2}\log \frac{2}{\delta}}{h^{2\alpha-\frac{1}{2}}}\rp\leq e\,.\label{eq_bd_exp_init}
    \ee
    Notice that $h>h_0^{1/(4\alpha-1)}$ also ensures $h>h_0^{1/(4\alpha-1)}\geq h_0\geq h_0'$, hence the width condition for \eqref{eq_thm_asymp_1}\eqref{eq_thm_asymp_2}\eqref{eq_thm_asymp_3} to hold is satisfied.
    
    Finally by \eqref{eq_thm_asymp_2}\eqref{eq_bd_exp_init}, we write \eqref{eq_bd_z_t_interm} as \begin{align*}
        \|Z(t)\|_F &\leq\;\lhp\exp\lp 1+ \frac{\lambda_1^{1/2}(\Sigma_x)}{\lambda_r(\Sigma_x)}\|Y\|_F\rp\rhp^{1/h^{1-2\alpha}} \|Z(0)\|_F\\
        &\leq\; \underbrace{\lhp\exp\lp 1+ \frac{\lambda_1^{1/2}(\Sigma_x)}{\lambda_r(\Sigma_x)}\|Y\|_F\rp\rhp^{1/h^{1-2\alpha}} }_{:=C^{1/h^{1-2\alpha}}}2\sqrt{m+r}\frac{\sqrt{m+n}+\frac{1}{2}\log \frac{2}{\delta}}{h^{2\alpha-\frac{1}{2}}}\\
        &=\;2C^{1/h^{1-2\alpha}}\sqrt{m+r}\frac{\sqrt{m+n}+\frac{1}{2}\log \frac{2}{\delta}}{h^{2\alpha-\frac{1}{2}}}\,.
    \end{align*}
    Therefore for some $C>0$ that depends on the data $(X,Y)$, given any $0<\delta<1$, when $h>h_0^{1/(4\alpha-1)}$ as defined above, with at least probability $1-\delta$, we have
    \begin{align*}
        \|U(\infty)V^T(\infty)-\hat{\Theta}\|_2&\leq\;\|U_2(0)V^T(\infty)\|_F\\
        &\leq\; \sup_{t>0}\|U_2(0)V^T(t)\|_F\\
        &\leq \;\sup_{t>0}\|Z(t)\|_F\leq 2C^{1/h^{1-2\alpha}}\sqrt{m+r}\frac{\sqrt{m+n}+\frac{1}{2}\log \frac{2}{\delta}}{h^{2\alpha-\frac{1}{2}}}\,.
    \end{align*}
\end{proof}



\end{document}